\documentclass{article}

\usepackage[preprint]{neurips_2025}

\usepackage[utf8]{inputenc} 
\usepackage[T1]{fontenc}    
\usepackage{hyperref}       
\usepackage{url}            
\usepackage{booktabs}       
\usepackage{amsfonts}       
\usepackage{nicefrac}       
\usepackage{microtype}      
\usepackage{xcolor}         
\usepackage{amsmath}
\usepackage{amsthm}
\usepackage{amssymb}
\usepackage{mathtools}
\usepackage{dsfont}
\usepackage{graphicx}
\usepackage{booktabs} 
\usepackage{bbm}
\usepackage{tikz}
\usetikzlibrary{shapes.geometric, calc, arrows.meta, patterns, backgrounds, positioning, shadows}

\theoremstyle{plain}
\newtheorem{theorem}{Theorem}[section]
\newtheorem{proposition}[theorem]{Proposition}
\newtheorem{lemma}[theorem]{Lemma}
\newtheorem{corollary}[theorem]{Corollary}
\theoremstyle{definition}
\newtheorem{definition}[theorem]{Definition}
\newtheorem{assumption}[theorem]{Assumption}
\theoremstyle{remark}
\newtheorem{remark}[theorem]{Remark}
\newtheorem{example}[theorem]{Example}

\title{Sparsity is Combinatorial Depth: Quantifying MoE Expressivity via Tropical Geometry}

\author{%
  Ye Su\\
  Shenzhen Institutes of Advanced Technology\\
  Chinese Academy of Sciences\\
  Shenzhen, 518055, China\\
  \texttt{ye.su@siat.ac.cn} \\
  \And
  Huayi Tang \\
  Gaoling School of Artificial Intelligence\\
  Renmin University of China\\
  Beijing, 100872,China\\
  \texttt{tangh4681@gmail.com} \\
  \And
  Zixuan Gong \\
  Gaoling School of Artificial Intelligence\\
  Renmin University of China\\
  Beijing, 100872, China\\
  \texttt{zxgong@ruc.edu.cn} \\
  \And
  Yong Liu\thanks{Corresponding author.}\\
  Gaoling School of Artificial Intelligence\\
  Renmin University of China\\
  Beijing, 100872, China\\
  \texttt{liuyonggsai@ruc.edu.cn} \\
}

\begin{document}

\maketitle

\begin{abstract}
    While Mixture-of-Experts (MoE) architectures define the state-of-the-art, their theoretical success is often attributed to heuristic efficiency rather than geometric expressivity. In this work, we present the first analysis of MoE through the lens of tropical geometry, establishing that the Top-$k$ routing mechanism is algebraically isomorphic to the $k$-th elementary symmetric tropical polynomial. This isomorphism partitions the input space into the Normal Fan of a Hypersimplex, revealing that \textbf{sparsity is combinatorial depth} which scales geometric capacity by the binomial coefficient $\binom{N}{k}$. Moving beyond ambient bounds, we introduce the concept of \textit{Effective Capacity} under the Manifold Hypothesis. We prove that while dense networks suffer from capacity collapse on low-dimensional data, MoE architectures exhibit \textit{Combinatorial Resilience}, maintaining high expressivity via the transversality of routing cones. Translating these theoretical bounds into architectural principles, we derive asymptotic capacity limits for optimal expert granularity and prove that shared experts are geometrically necessary to prevent routing collapse.
\end{abstract}

\section{Introduction}

Mixture-of-Experts (MoE) architectures have emerged as a prominent approach for scaling neural networks by conditionally activating a small subset of parameters for each input \citep{cai2025survey,mu2025comprehensive}. By decoupling total parameter count from per-input computation, known as conditional computation, MoE models enable the deployment of extremely large networks under fixed inference budgets and have become a standard component in modern large-scale systems \citep{jiang2024mixtral,dai2024deepseekmoe}.

At first glance, however, the success of MoE appears to challenge the conventional premise in learning theory: reducing the number of active parameters should restrict a model’s expressive power. Classical analyses of neural network expressivity are traditionally tied to architectual \textit{width} and \textit{depth} \citep{pascanu2013number,montufar2014number}. Under this paradigm, sparse models appear inherently limited as their active computation utilizes only a fraction of total capacity~\citep{fedus2022switch,su2026variational}. This tension raises a central theoretical question: \textbf{\textit{why does sparse routing in MoE not degrade expressivity, and in fact often enhances it?}} More precisely, \textit{\textbf{in what formal sense does conditional computation expand the expressive power of a neural network, despite activating fewer parameters per input?}}

To address this question, we analyze the expressive power of MoE by quantifying the number of linear regions it induces. For dense networks, capacity is typically computed using the theory of affine hyperplane arrangements, which assumes hyperplanes act globally across the input space \citep{raghu2017expressive,serra2018bounding,hanin2019deep,xiong2020number,xiong2024number}. However, evaluating MoE expressivity poses a severe mathematical challenge. The Top-$k$ routing mechanism is a data-dependent discrete sorting operation. It dynamically truncates the domain of expert hyperplanes into local, non-convex polyhedral cones. Standard geometric counting tools rely on global intersection properties and collapse when confronted with the combinatorial explosion of nested inequalities required to track these local, conditional hyperplane fragments. Therefore, existing frameworks for dense architectures cannot capture the topology of input-dependent routing. This gap leads to the central question of this work: \textbf{\textit{How does sparse routing quantitatively expand the expressivity of a neural network?}} 

Answering this requires a framework capable of reasoning not only about piecewise-linear functions, but also about the combinatorial structure induced by input-dependent expert selection. Tropical geometry, which studies algebraic structures over the max-plus semiring, offers an exact language for modeling maximization, ordering, and piecewise-affine behavior \citep{zhang2018tropical,maragos2021tropical,pham2024graph,fotopoulos2024tropnnc,brandenburg2024real}. These operations are essential for a formal treatment of Top-$k$ routing in MoE. Therefore, tropical geometry is \textbf{not merely} a descriptive tool, but an \textbf{algebraic necessity} for analyzing conditional computation. While prior tropical analyses of neural networks have focused on dense architectures, we demonstrate that Top-k routing in MoE admits a canonical tropical formulation, revealing a previously unrecognized combinatorial structure underlying sparse computation.

In this paper, we present the first analysis of MoE expressivity derived from the perspective of tropical geometry. We propose that sparsity is not merely a systems optimization to reduce FLOPs, but a fundamental topological shift in the function space. By mapping the decision boundaries of the router to the Normal Fan of a \textbf{Hypersimplex}, we prove that \textbf{Sparsity is Combinatorial Depth}: a mechanism that scales geometric capacity by the binomial coefficient $\binom{N}{k}$, effectively bypassing the polynomial scaling limits of dense width. Our contributions are summarized as follows:
\begin{itemize}
    \item \textbf{Tropical Isomorphism of Routing:} We establish the first algebraic isomorphism between Top-$k$ routing and the $k$-th elementary symmetric tropical polynomial, showing that routing partitions correspond to the Normal Fan of a hypersimplex and yielding a precise geometric characterization of sparse routing (Section~\ref{subsec:geometry_routing}).

    \item \textbf{The Combinatorial Slicing Theorem:} We derive matching upper and lower bounds on the number of linear regions realized by MoE networks, thereby characterizing the expressivity up to tight polynomial order. In particular, we prove that the capacity necessarily scales as $\Theta\!\left(\binom{N}{k}(kH)^{d_{in}}\right)$, establishing the first tight expressivity for MoE architectures (Section~\ref{subsec:capacity_analysis}).

    \item \textbf{Effective Capacity on Manifolds:} We show that MoE models preserve high expressivity on low-dimensional data manifolds (\textit{Combinatorial Resilience}), while dense networks exhibit geometric capacity collapse, revealing a fundamental advantage of input-dependent routing beyond efficiency (Section~\ref{subsec:effective_capacity}).

    \item \textbf{Architectural Laws for Optimal Sparsity:} We translate our geometric framework into prescriptive design principles. We derive an asymptotic capacity expansion for \textbf{Fine-Grained Experts}, bounded by a critical granularity limit $\left\lfloor \frac{H}{d_{eff}} \right\rfloor$ to prevent \textit{Tropical Rank Deficiency}. Furthermore, we prove that \textbf{Shared Experts} are geometrically necessary as \textit{Affine Rectifiers}, stabilizing the routing mechanism against \textbf{Angular Collapse} (Section~\ref{sec:architectural_laws}).
\end{itemize}

\section{Related Work}
\label{sec:related_work}

\subsection{Engineering Design and Theoretical Understanding of Mixture-of-Experts}
The concept of MoE was formalized by \citet{jacobs1991adaptive} as a method to decouple sub-tasks via gating. In the context of deep learning, \citet{shazeer2017outrageously} reintroduced MoE to LSTM networks, demonstrating that conditional computation could scale model capacity by orders of magnitude without increasing inference FLOPs. This sparsity paradigm was subsequently integrated into the Transformer architecture \citep{vaswani2017attention}. Key adaptations include GShard \citep{lepikhin2020gshard}, which enabled large-scale parallel training, and the Switch Transformer \citep{fedus2022switch}, which simplified routing to scale model capacity beyond a trillion parameters. Recent advancements focus on optimizing granularity and routing dynamics, notably through Top-2 routing in Mixtral \citep{jiang2024mixtral} and fine-grained experts in DeepSeek-MoE \citep{dai2024deepseekmoe}. 

Theoretical analysis of MoE has evolved from empirical scaling laws \citep{kaplan2020scaling, clark2022unified} to rigorous statistical and optimization frameworks. \citet{chen2022towards} provided a fundamental analysis of the MoE layer's properties in deep learning. A significant body of recent work by \citet{nguyen2023general, nguyen2024towards, nguyen2025convergence} has established the statistical convergence rates for various gating mechanisms, including softmax and Gaussian-gated models. Further refinements include the statistical advantages of cosine routers \citep{nguyen2024statistical} and the properties of quadratic gating functions \citep{akbarian2024quadratic}. Most recently, \citet{su2026variational} proposed a unified theory of MoE incorporating variational inference, entropy, and orthogonality. Despite these insights into parameter estimation and optimization, the \textit{geometric expressivity} of MoE, specifically how sparse coalitions partition the input space, remains an under-explored dimension. \citet{chi2022representation} notably identified representation collapse as a primary challenge in training these dynamic architectures.

\subsection{Complexity of Linear Regions}
A standard measure of expressivity for ReLU networks is the number of linear regions induced in the input space. \citet{pascanu2013number} and \citet{montufar2014number} established that the number of regions scales exponentially with network depth but polynomially with width. \citet{raghu2017expressive} introduced the trajectory length measure, confirming the exponential influence of depth. To tighten these bounds, \citet{serra2018bounding} and \citet{hinz2019framework} utilized the theory of hyperplane arrangements \citep{zaslavsky1975facing} to derive exact upper bounds for deterministic networks. This combinatorial framework was further extended to more complex architectures by \citet{xiong2020number} and \citet{xiong2024number}, who provided rigorous counts for the number of linear regions in convolutional neural network with various pooling layers and piecewise linear activations. 
Crucially, \citet{hanin2019deep} also investigated the effective capacity restricted to low-dimensional manifolds, proving that while the ambient region count is exponential, the number of regions intersecting a 1D curve scales only linearly with the number of neurons. This distinction between ambient and effective capacity is central to understanding generalization on real-world data. Despite these advances, a notable limitation remains: existing results are predominantly restricted to static, dense architectures where the network topology is fixed. Consequently, the unique geometric paradigm of MoE, characterized by the combinatorial interplay between a data-dependent routing mechanism and the local arrangement complexity of sparse expert sub-networks, remains uncharacterized in the current literature of arrangement theory or neural complexity. Specifically, how the router's selection of expert coalitions interacts with the piecewise linear boundaries of the experts to expand the global capacity remains a critical theoretical gap.

\subsection{Tropical Geometry for Neural Networks}
Tropical geometry provides a natural algebraic framework for analyzing piecewise linear functions. \citet{zhang2018tropical} formally established the isomorphism between Feedforward ReLU networks and Tropical Rational Functions. This foundation has been expanded by \citet{charisopoulos2018tropical} and \citet{maragos2021tropical} to model general morphological perceptrons. Recent work has pushed this boundary significantly. \citet{alfarra2022decision} utilized tropical algebra to characterize the decision boundaries of classifiers. \citet{brandenburg2024real} advanced the theoretical understanding by developing "real" tropical geometry for neural networks, rigorously handling signed weights in binary classification. On the application front, \citet{pham2024graph} extended tropical insights to interpret the learning dynamics of Graph Neural Networks (GNNs), while \citet{fotopoulos2024tropnnc} proposed TropNNC for structured neural network compression. 
However, these works predominantly analyze static architectures (Dense MLPs, GNNs). The specific \textit{combinatorial geometry} of conditional computation (MoE), where the network topology itself changes dynamically via Top-$k$ routing, and its connection to the Tropical Grassmannian \citep{speyer2003tropical} remain an open frontier.

\section{Preliminaries: Neural Networks as Tropical Rational Functions}
\label{sec:preliminaries}

In this section, we establish the algebraic foundation for our analysis. We first recall the basics of tropical geometry and the representation of standard dense networks as tropical rational functions. We then formally introduce the MoE architecture and derive its tropical representation, demonstrating that the Top-$k$ routing mechanism corresponds to evaluating elementary symmetric tropical polynomials.

\subsection{Expressivity via Linear Regions}
\label{subsec:expressivity_definition}

To quantify the expressive power of neural networks, we adopt the widely accepted metric of \textit{linear regions} \citep{pascanu2013number,montufar2014number}. Deep neural networks with piecewise-linear (PWL) activations, such as the Rectified Linear Unit (ReLU), compute continuous PWL functions. The input space is partitioned into a finite set of convex polyhedral cells, within which the network acts as a purely affine transformation.

\begin{definition}[Linear Regions and Geometric Capacity]
\label{def:linear_regions}
Let $F: \mathbb{R}^{d_{in}} \to \mathbb{R}^{d_{out}}$ be a continuous piecewise-linear neural network, where $d_{in}$ and $d_{out}$ denote the input and output dimensions, respectively. A \textbf{linear region} is a maximal connected open subset $\mathcal{R} \subset \mathbb{R}^{d_{in}}$ such that the restriction $F|_{\mathcal{R}}$ is an affine function \citep{montufar2014number}. The \textbf{Geometric Capacity} (or Expressivity) of the architecture, denoted as $\mathcal{C}(F)$, is defined as the maximum number of such disjoint linear regions the network can induce over the input space $\mathbb{R}^{d_{in}}$, taken over all possible parameter instantiations \citep{serra2018bounding}.
\end{definition}

The number of linear regions measures the geometric flexibility of the model: a higher capacity allows the network to approximate highly non-convex functions with finer piecewise-affine boundaries. 

\begin{definition}[Zaslavsky's Function]
\label{def:zaslavsky}
Let $\Phi(n, d)$ denote the number of regions induced by an arrangement of $n$ hyperplanes in general position in a space of dimension $d$ \citep{zaslavsky1975facing}:
\begin{equation*}
    \Phi(n, d) \coloneqq \sum_{j=0}^{d} \binom{n}{j}.
\end{equation*}
For $n \gg d$, this sum is dominated by the highest-order term: $\Phi(n, d) \approx \frac{1}{d!} n^d$.
\end{definition}

\begin{example}[1-layer and 2-layer ReLU MLP of linear regions]
Consider a 2-layer ReLU MLP mapping $\Phi: \mathbb{R}^{d_{in}} \to \mathbb{R}^{d_{out}}$ with ambient dimension $d_{in}=2$ and $n=3$ neurons per layer. According to Definition \ref{def:zaslavsky}, the first hidden layer induces exactly $N_{mlp1} = \sum_{j=0}^{d_{in}} \binom{n}{j} = \binom{3}{0} + \binom{3}{1} + \binom{3}{2} = 7$ maximal linear regions. In Figure~\ref{fig:mlp_expressivity}, deep composition facilitates recursive shattering: the second layer's hyperplanes intersect the folded image of $\mathbb{R}^{d_{in}}$, subdividing each pre-existing region. This multiplicative interaction yields a theoretical maximum of $N_{mlp2} = N_{mlp1} \cdot \sum_{j=0}^{d_{in}} \binom{n}{j} = 49$ regions. This geometric escalation provides a classical measure of expressivity, which this work extends to the non-piecewise-linear self-attention mechanism.
\end{example}

\begin{figure}
    \centering
    \includegraphics[width=0.75\linewidth]{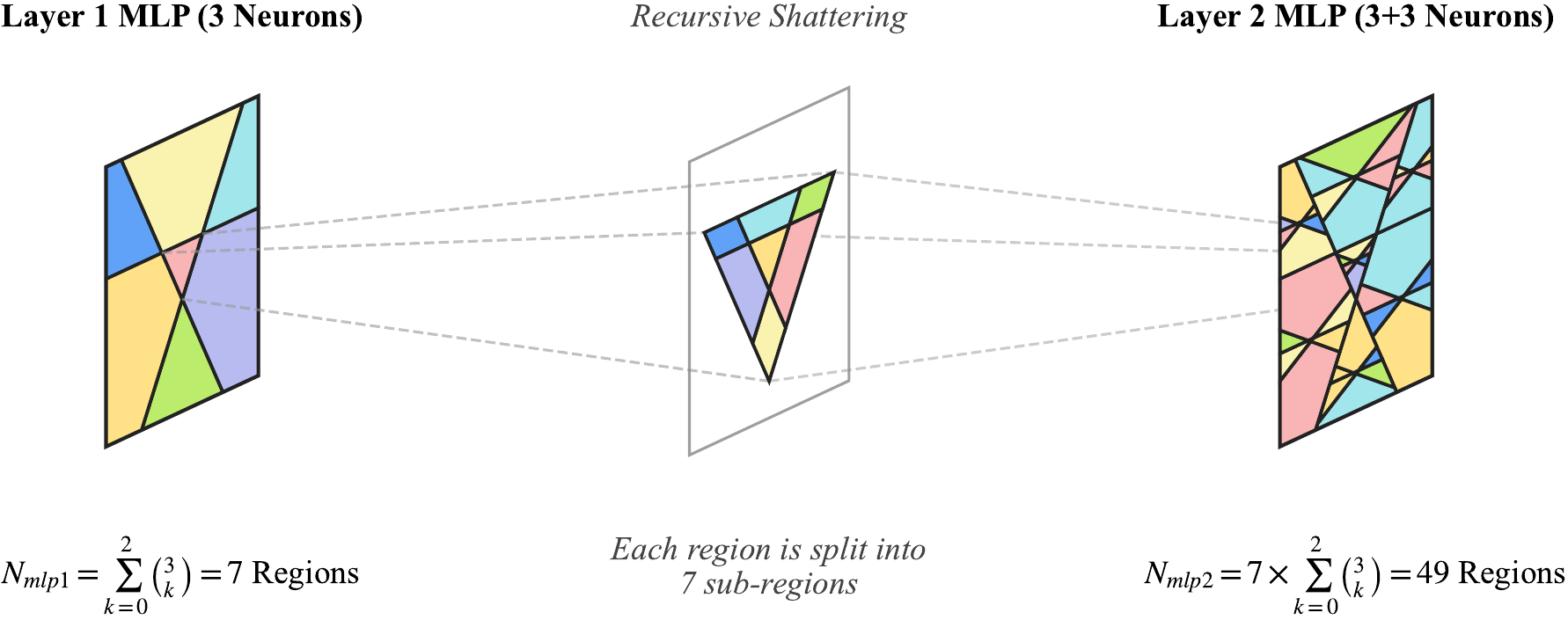}
    \caption{\textbf{Recursive spatial partitioning in a 2-layer MLP ($d=2$).} \textbf{(Left)} Layer 1 ($n=3$ neurons) induces $N_{mlp1} = 7$ linear regions via Definition \ref{def:zaslavsky}. \textbf{(Middle)} Structural space-folding: the ReLU activation enables subsequent hyperplanes to intersect the pre-activated regions. \textbf{(Right)} Layer 2 ($n=3$) achieves a multiplicative expansion, yielding a theoretical maximum of $N_{mlp2} = 49$ regions.}
    \label{fig:mlp_expressivity}
    \vspace{-0.3cm}
\end{figure}

\subsection{The Tropical Semiring}

Tropical geometry, often described as algebraic geometry over the max-plus semiring, provides a natural language for analyzing piecewise linear functions involving maximization operations \citep{mikhalkin2006tropical,maclagan2015introduction,maragos2021tropical}.

\begin{definition}[Tropical Semiring]
The tropical semiring, denoted as $\mathbb{T}$, is the set $\mathbb{R} \cup \{-\infty\}$ equipped with two binary operations: 1) \textbf{Tropical Addition ($\oplus$):} $x \oplus y := \max(x, y)$. 2) \textbf{Tropical Multiplication ($\otimes$):} $x \otimes y := x + y$.
The identity element for addition is $-\infty$ (denoted as $\mathbbm{O}$), and the identity element for multiplication is $0$ (denoted as $\mathbbm{1}$).
\end{definition}

In this algebraic structure, exponentiation corresponds to standard multiplication: $x^{\otimes k} := k \cdot x$ for $k \in \mathbb{N}$. A \textbf{tropical monomial} in $d$ variables $\mathbf{x} = (x_1, \dots, x_d)^\top$ is an expression of the form $c \otimes x_1^{\otimes a_1} \otimes \dots \otimes x_d^{\otimes a_d}$, which translates to the affine function $c + \sum_{j=1}^d a_j x_j$ in standard arithmetic.

\begin{definition}[Tropical Polynomial]
A tropical polynomial $P: \mathbb{R}^d \to \mathbb{R}$ is the tropical sum of a finite number of tropical monomials:
\begin{equation*}
    P(\mathbf{x}) = \bigoplus_{i=1}^m \left( c_i \otimes \bigotimes_{j=1}^d x_j^{\otimes a_{ij}} \right) = \max_{i=1,\dots,m} \left( c_i + \langle \mathbf{a}_i, \mathbf{x} \rangle \right),
\end{equation*}
where $\mathbf{a}_i \in \mathbb{N}^d$, regarded as points in $\mathbb{R}^d$ \citep{zhang2018tropical} and $c_i \in \mathbb{R}$. Geometrically, a tropical polynomial defines a convex function as the pointwise maximum of affine functions.
\end{definition}

The \textbf{tropical hypersurface} defined by $P(\mathbf{x})$ is the locus of points where the maximum is attained by at least two distinct terms \citep{maclagan2015introduction}. This corresponds to the non-differentiable \textit{creases} of the convex function, dividing the domain into linear regions (polyhedral cells) \citep{alfarra2022decision}.

\subsection{Neural Networks as Tropical Rational Functions}

Before analyzing MoE, we briefly characterize the geometry of standard dense networks. It is well-established that the Rectified Linear Unit (ReLU), $\sigma(x) = \max(0, x)$, is a native tropical operation: $\sigma(x) = \mathbbm{1} \oplus x$. However, neural networks fundamentally rely on negative weights to approximate non-convex functions \citep{yarotsky2017error}. Since the tropical semiring lacks an additive inverse (subtraction is undefined) \citep{viro2011basic,della2017tropical}. To address this, prior works \citep{zhang2018tropical,charisopoulos2018tropical} model neural networks as \textbf{Tropical Rational Functions}. 

\begin{definition}[Tropical Rational Function]
A tropical rational function $F: \mathbb{R}^d \to \mathbb{R}$ represents a dense feedforward network if it can be expressed as the formal difference of two tropical polynomials (convex functions) $P(\mathbf{x})$ and $Q(\mathbf{x})$ \citep{song2023congruences,sukenaga2024minimum}:
\begin{equation*}
    F(\mathbf{x}) = P(\mathbf{x}) - Q(\mathbf{x}) \quad \text{(Standard Arithmetic)} \iff F(\mathbf{x}) = P(\mathbf{x}) \oslash Q(\mathbf{x}) \quad \text{(Tropical Notation)}.
\end{equation*}
Here, $Q(\mathbf{x})$ acts as a \textit{normalization potential} introduced by negative weights. The decision boundaries of a dense network are subsets of the tropical hypersurface defined by $P(\mathbf{x}) \oplus Q(\mathbf{x})$. While this framework unifies dense networks, the geometric complexity is implicitly folded into the depth of $P$ and $Q$. As we will show, MoE architectures introduce complexity through a fundamentally different mechanism: explicit combinatorial partitioning.
\end{definition}

\subsection{Mixture-of-Experts: Formulation and Gating}
\label{sec:moe_formulation}

\begin{figure}
    \centering
    \begin{tikzpicture}[
        scale=0.8, 
        transform shape, 
        >=Stealth, 
        node distance=1.2cm and 1.0cm, 
        font=\footnotesize\rmfamily
    ]
        \tikzstyle{expert} = [
            draw=blue!40!black, 
            fill=blue!5, 
            very thick,
            rectangle, 
            rounded corners=2pt, 
            minimum width=1.3cm, 
            minimum height=0.8cm, 
            align=center
        ]
        
        \tikzstyle{router} = [
            draw=orange!60!black, 
            fill=orange!10, 
            thick,
            trapezium, 
            trapezium left angle=75, 
            trapezium right angle=105, 
            minimum width=1.6cm,
            minimum height=0.6cm,
            align=center
        ]
        
        \tikzstyle{sum} = [
            draw, 
            circle, 
            fill=white, 
            minimum size=0.6cm, 
            inner sep=0pt,
            thick
        ]
        
        \node[align=center] (input) {Input $\mathbf{x} \in \mathbb{R}^{d_{in}}$};
        \node[router, above=1.0cm of input] (router) {Router $G(\mathbf{x})$};
        
        \node[above=1.4cm of router] (dots) {$\dots$};
        \node[expert, right=0.2cm of dots] (e2) {Expert $E_{k}$};
        \node[expert, right=0.6cm of e2] (en) {Expert $E_{N}$};
        \node[expert, left=0.6cm of dots, draw=gray!30, fill=gray!5, text=gray!60] (e1) {Expert $E_{1}$};
        
        \node[sum, above=0.9cm of dots] (sum) {$\Sigma$};
        \node[align=center, above=0.5cm of sum] (output) {Output $\mathbf{y}$};
        
        \draw[->, thick] (input) -- (router);
        
        \draw[->, dashed, gray!40] (input.north) to[out=140, in=270] (e1.south);
        \draw[->, very thick, blue!80!black] (input.north) to[out=60, in=270] (e2.south); 
        \draw[->, very thick, blue!80!black] (input.north) to[out=40, in=270] (en.south);

        \draw[->, dashed, gray!40] (router.north west) -- (e1.south east) 
            node[midway, fill=white, inner sep=0.5pt, font=\scriptsize, text=gray!60] {$G_1 \approx 0$};
        \draw[->, thick, red!80!black] (router.north) -- (e2.south west) 
            node[midway, fill=white, inner sep=0.5pt, font=\scriptsize] {$G_{k}$};
        \draw[->, thick, red!80!black] (router.east) -- (en.south) 
            node[midway, fill=white, inner sep=0.5pt, font=\scriptsize] {$G_{N}$};

        \draw[->, dashed, gray!40] (e1.north) -- (sum);
        \draw[->, thick] (e2.north) -- (sum);
        \draw[->, thick] (en.north) -- (sum);
        \draw[->, thick] (sum) -- (output);
        
        \node[draw=black!10, fill=yellow!2, rounded corners=1pt, align=left, font=\scriptsize, right=0.3cm of en, anchor=west] (legend) {
            \textbf{Key Concepts:}\\
            \textcolor{blue!80!black}{\rule[0.5ex]{0.3cm}{1.5pt}} Data Flow (Input)\\
            \textcolor{red!80!black}{\rule[0.5ex]{0.3cm}{1.5pt}} Control Flow (Gating)\\
            \textcolor{gray}{\textbf{-- --}} Inactive Path
        };

    \end{tikzpicture}
    \caption{\textbf{Schematic of a Top-$k$ Mixture-of-Experts Layer.} The architecture separates \textcolor{red!80!black}{routing logic} from \textcolor{blue!80!black}{data processing}. The router selects a sparse subset of experts (Active Set $\mathcal{I}_\mathbf{x}$), ensuring that inactive experts (grey, dashed) consume no computational resources. The final output is the weighted sum of the active experts.}
    \label{fig:moe_arch}
    \vspace{-0.4cm}
\end{figure}

An MoE layer replaces the standard dense feed-forward block in neural networks with a conditional computation mechanism \citep{lou2021cross}. Formally, let $\mathbf{x} \in \mathbb{R}^{d_{in}}$ be the input vector (e.g., a token representation in a Transformer). The layer consists of a bank of $N$ \textbf{expert networks} $\{E_1, \dots, E_N\}$, where each expert acts as a function $E_i: \mathbb{R}^{d_{in}} \to \mathbb{R}^{d_{out}}$, and a \textbf{router} (or gating network) $R: \mathbb{R}^{d_{in}} \to \mathbb{R}^N$ \citep{cai2025survey, mu2025comprehensive}. The output $\mathbf{y} \in \mathbb{R}^{d_{out}}$ is computed as the sparse weighted sum of the expert outputs \citep{su2026variational}:
\begin{equation*}
    \mathbf{y} = \sum_{i=1}^N G(\mathbf{x})_i E_i(\mathbf{x}),
\end{equation*}
where $G(\mathbf{x}) \in \mathbb{R}^N$ is the sparse gating vector derived from the router. Unlike dense networks that utilize a single monolithic weight matrix, an MoE layer distributes its capacity across $N$ independent experts. Each expert $E_i$ is typically parameterized as a two-layer piecewise-linear network (e.g., a ReLU MLP). Importantly, we denote the hidden layer width of each expert as $H$. 

\paragraph{Sparse Top-$k$ Routing Mechanism.} 
As illustrated in Figure~\ref{fig:moe_arch}, modern MoE architectures (e.g., Mixtral, DeepSeek-MoE) employ a sparse gating mechanism to fundamentally decouple model capacity from inference cost \citep{jiang2024mixtral,dai2024deepseekmoe}. The gating vector is sparse, meaning only $k \ll N$ entries are non-zero. The routing process initiates by projecting the input $\mathbf{x}$ onto a learnable embedding space to generate an input-dependent routing logit vector $\mathbf{z}(\mathbf{x}) \in \mathbb{R}^N$: $\mathbf{z}(\mathbf{x}) = \mathbf{W}_r \mathbf{x} + \mathbf{b}_r$, where $\mathbf{W}_r \in \mathbb{R}^{N \times d_{in}}$ and $\mathbf{b}_r \in \mathbb{R}^N$ are the router's weights and bias. Consequently, the logit score for the $i$-th expert is an affine function of the input, denoted as $z_i(\mathbf{x}) = \mathbf{w}_{r,i}^\top \mathbf{x} + b_{r,i}$ (where $\mathbf{w}_{r,i}^\top$ is the $i$-th row of $\mathbf{W}_r$). From these logits, the router selects a sparse coalition of experts by identifying the indices corresponding to the $k$ largest elements of $\mathbf{z}(\mathbf{x})$. Formally, the active set $\mathcal{I}_{\mathbf{x}}$ is defined as:
\begin{equation}
\label{eq:active_set}
    \mathcal{I}_{\mathbf{x}} = \text{argtop}_k(\mathbf{z}(\mathbf{x})) = \{i \mid z_i(\mathbf{x}) \text{ is in the top-}k \text{ values of } \mathbf{z}(\mathbf{x})\}.
\end{equation}
Once the active experts are identified, their specific contribution weights are computed via a Softmax normalization restricted to the active set $\mathcal{I}_\mathbf{x}$:
\begin{equation}
\label{eq:gating_norm}
    G(\mathbf{x})_i = \begin{cases} 
    \frac{\exp(z_i(\mathbf{x}))}{\sum_{j \in \mathcal{I}_\mathbf{x}} \exp(z_j(\mathbf{x}))} & \text{if } i \in \mathcal{I}_\mathbf{x}, \\
    0 & \text{otherwise}.
    \end{cases}
\end{equation}
Importantly, this mechanism ensures that the computational cost of the forward pass scales with $k$ rather than $N$. In our geometric analysis, we focus specifically on the decision boundaries induced by the discrete selection operator in Eq.~\eqref{eq:active_set}. While the Softmax normalization (Eq.~\eqref{eq:gating_norm}) determines the continuous magnitude of the output, the fundamental partitioning of the input space is governed exclusively by the Top-$k$ operator.

\paragraph{Complexity Analysis: Active vs. Total Parameters.} 
To evaluate the expressivity of MoE, we must establish an \textit{iso-compute constraint} comparing it to standard dense networks. We distinguish between two types of parameter counts (ignoring biases and output layers for asymptotic clarity):
\begin{itemize}
    \item \textbf{Total Parameters ($P_{total}$: Memory Cost).} A dense model with hidden width $H_{dense}$ has $P_{total} = H_{dense} \cdot d_{in}$. An MoE model with $N$ experts of width $H$ stores all experts in memory, yielding a total parameter count of $P_{total} = N \cdot H \cdot d_{in}$.
    \item \textbf{Active Parameters ($P_{active}$: Inference FLOPs).} For a given input $\mathbf{x}$, the dense model activates all neurons, so $P_{active} = H_{dense} \cdot d_{in}$. In contrast, the Top-$k$ MoE model conditionally activates only $k$ experts. Therefore, its inference cost is strictly bounded by $P_{active} = k \cdot H \cdot d_{in}$ (plus a negligible routing overhead).
\end{itemize}

To ensure a fair comparison of geometric expressivity, we control the computational variable by matching the active parameters of the Dense and MoE models:
\begin{equation*}
    P_{active}^{Dense} = P_{active}^{MoE} \iff H_{dense} = k \cdot H.
\end{equation*}
Under this constraint, both models perform the exact same number of FLOPs per forward pass. However, the MoE model possesses a vastly larger parameter pool ($N \cdot H \cdot d_{in} \gg H_{dense} \cdot d_{in}$ since $N \gg k$). The central theoretical question of this paper is: \textbf{\textit{How does the routing mechanism leverage this exponentially larger, yet sparsely activated, parameter pool to expand the network's expressive geometry?}}

\subsection{The Tropical Geometry of Routing}
We now present our primary formulation: the interpretation of the MoE routing mechanism as a tropical algebraic operation. This perspective allows us to quantify the \textbf{combinatorial depth} of the model.

\textbf{Top-1 Routing (Tropical Maximization).} 
In the simplest case ($k=1$), the router selects the single expert with the maximum logit. The score of the selected expert is given by the tropical sum of the logits:
\begin{equation*}
    S_{\text{top1}}(\mathbf{x}) = \bigoplus_{i=1}^N z_i(\mathbf{x}) = \max_{i=1,\dots,N} (\mathbf{w}_{i}^{r\top} \mathbf{x} + b_{i}^r).
\end{equation*}
This is a tropical polynomial of degree 1. Its singular locus, $\mathcal{T}(S_{\text{top1}})$, forms a classical Voronoi diagram (or power diagram) partitioning $\mathbb{R}^d$ into $N$ convex polytopes.

\textbf{Top-k Routing (Combinatorial Tropical Polynomials).} 
When $k > 1$, the router selects a coalition of experts that maximizes the aggregate score. Algebraically, the score of the optimal coalition corresponds to the \textbf{$k$-th Elementary Symmetric Tropical Polynomial} \citep{verovsek2014symmetric}, denoted as $\text{SymTrop}_k$:
\begin{equation*}
    S_{\text{topk}}(\mathbf{x}) = \text{SymTrop}_k(z_1, \dots, z_N) := \bigoplus_{\substack{\mathcal{I}_\mathbf{x} \subseteq \{1,\dots,N\} \\ |\mathcal{I}_\mathbf{x}|=k}} \left( \bigotimes_{j \in \mathcal{I}_\mathbf{x}} z_j(\mathbf{x}) \right).
\end{equation*}
Translating back to standard arithmetic, this becomes:
\begin{equation*}
    S_{\text{topk}}(\mathbf{x}) = \max_{\substack{\mathcal{I}_\mathbf{x} \subset \{1,\dots,N\} \\ |\mathcal{I}_\mathbf{x}|=k}} \sum_{j \in \mathcal{I}_\mathbf{x}} (\mathbf{w}_{j}^{r\top} \mathbf{x} + b_{j}^r).
\end{equation*}
This formulation provides the rigorous basis for the capacity analysis in the subsequent section.

\section{Geometric Structure and Capacity Analysis}
\label{sec:capacity}

In this section, we quantify the expressivity of MoE. Our analysis proceeds in three steps: 1) we characterize the \textit{shape} of the decision boundaries induced by the router; 2) we derive the \textit{Combinatorial Slicing Theorem}, which provides matching upper and lower bounds on the number of linear regions in the idealized setting; and 3) we refine these bound under the Manifold Hypothesis to derive the \textit{Effective Capacity}.

\subsection{The Geometry of Routing: From Hyperplanes to Hypersimplices}
\label{subsec:geometry_routing}

To understand the expressivity of MoE, we first contrast its geometric structure with that of standard dense networks. Consider a standard dense feedforward layer with width $H$. The pre-activation function consists of $H$ affine transformations $h_i(\mathbf{x}) = \mathbf{w}_i^\top \mathbf{x} + b_i$. Geometrically, the zeros of these functions define a set of hyperplanes in the input space $\mathbb{R}^{d_{in}}$.

\begin{definition}[Hyperplane Arrangement]
The geometry of a dense layer is characterized by a hyperplane arrangement $\mathcal{A} = \{ \mathcal{H}_1, \dots, \mathcal{H}_H \}$, where $\mathcal{H}_i = \{ \mathbf{x} \in \mathbb{R}^{d_{in}} \mid \mathbf{w}_i^\top \mathbf{x} + b_i = 0 \}$.
\end{definition}

This arrangement partitions the input space into convex polyhedral regions (cells). Within each region, the sign pattern of the activations is constant, and the network behaves linearly. Crucially, the complexity of this partition is coupled to the physical width $H$. To increase the topological complexity (number of regions), one must linearly scale the number of neurons, incurring a quadratic cost in parameters and FLOPs. While the dense layer relies on independent cuts, the MoE router introduces a competitive mechanism. For $k=1$, the router selects the maximum of $N$ affine logits. 

\begin{proposition}[Top-1 MoE Geometry: Tropical Hypersurfaces and Power Diagrams]
\label{prop:top1_geometry}
Let $\mathcal{N} = \{1, \dots, N\}$ be the set of expert indices, and let $z_i(\mathbf{x})$ denote the affine logit score for expert $i$. The decision boundaries of a Top-1 router form a \textbf{Tropical Hypersurface}, formally defined as the singular locus of the tropical polynomial $S_{\text{top1}}(\mathbf{x}) = \max_{k \in \mathcal{N}} z_k(\mathbf{x})$:
\begin{equation*}
    \mathcal{T}(S_{\text{top1}}) = \left\{ \mathbf{x} \in \mathbb{R}^{d_{in}} \;\middle|\; \exists i, j \in \mathcal{N}, i \neq j : z_i(\mathbf{x}) = z_j(\mathbf{x}) = S_{\text{top1}}(\mathbf{x}) \right\}.
\end{equation*}
Geometrically, this locus constitutes the boundaries of a \textbf{Power Diagram}, partitioning $\mathbb{R}^{d_{in}}$ into $N$ convex polyhedral cells \citep{aurenhammer1987power}.
\end{proposition}

\begin{example}[Analyzing the Singular Locus and Decision Boundaries.]
To concretize Proposition \ref{prop:top1_geometry}, consider the geometry presented in Figure \ref{fig:top1_geometry}. Let us track an input feature $\mathbf{x}$ traversing from the domain of Expert 5 ($\Omega_5$) into the domain of Expert 1 ($\Omega_1$).
\begin{itemize}
    \item \textbf{Inside the Cell:} While $\mathbf{x}$ is within $\Omega_5$ (the orange region), its logit score dominates: $z_5(\mathbf{x}) > z_k(\mathbf{x})$ for all $k \neq 5$. On the 3D tropical hypersurface (Figure \ref{fig:top1_geometry}, Right), this region corresponds to the smooth, flat facet representing the affine plane $z = z_5(\mathbf{x})$. In this state, the top-1 routing is stable and differentiable.
    
    \item \textbf{On the Boundary (Singular Locus):} As $\mathbf{x}$ hits the solid black boundary between $\Omega_5$ and $\Omega_1$ (Figure \ref{fig:top1_geometry}, Left), a routing tie occurs: $z_1(\mathbf{x}) = z_5(\mathbf{x}) = S_{\text{top1}}(\mathbf{x})$. Geometrically, this exact moment corresponds to the crease or valley on the 3D tropical hypersurface where two affine planes intersect. By definition in tropical geometry, this intersection constitutes the \textit{singular locus} $\mathcal{T}(S_{\text{top1}})$, where the maximum function becomes non-differentiable.
    
    \item \textbf{Voronoi Equivalence:} Because the affine logits $z_i(\mathbf{x})$ in our router implicitly act as scaled distance metrics to the expert centroids (black dots), the equation $z_1(\mathbf{x}) = z_5(\mathbf{x})$ simplifies to the hyperplane equation of a perpendicular bisector. This explains why the projection of the tropical singular locus recovers a classical Voronoi Diagram.
\end{itemize}
\end{example}

\begin{figure}
    \centering
    \includegraphics[width=\textwidth]{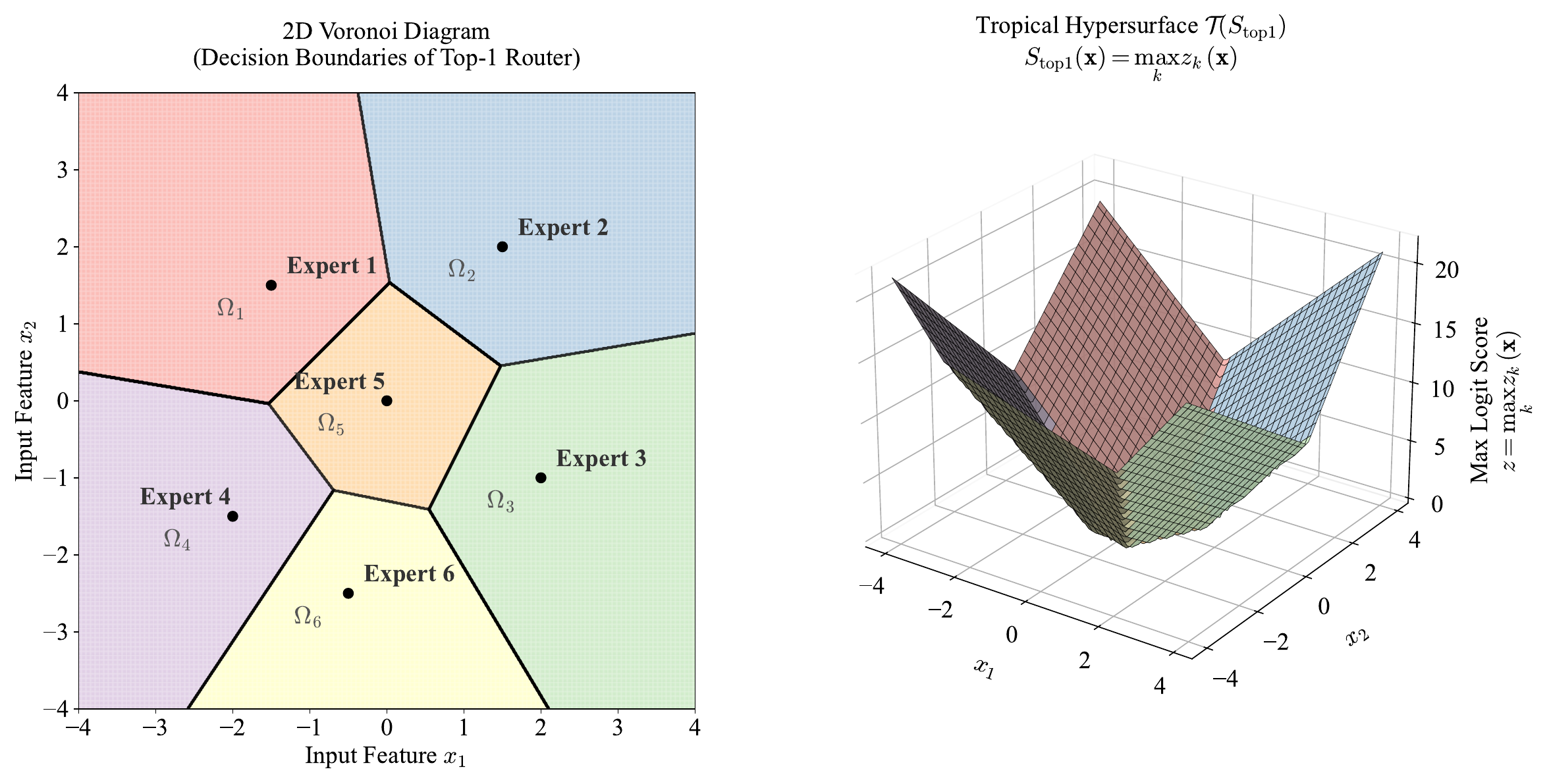}
    \caption{\textbf{Geometric interpretation of the Top-1 MoE router.} \textbf{(Left)} The router partitions the input space $\mathbb{R}^2$ into a Voronoi diagram, where each convex polyhedral cell $\Omega_i$ corresponds to the active domain of Expert $i$. \textbf{(Right)} The maximum logit score function $S_{\text{top1}}(\mathbf{x}) = \max_k z_k(\mathbf{x})$ constitutes a piecewise-linear tropical hypersurface. The singular locus of this hypersurface (its non-differentiable creases) orthogonally projects onto the exact Voronoi decision boundaries in the 2D plane.}
    \label{fig:top1_geometry}
    \vspace{-0.4cm}
\end{figure}

\begin{remark}[The Sparsity Gain]
    Unlike the dense case, where $H$ hyperplanes generate independent cuts, the $N$ experts in a Top-1 router define $N(N-1)/2$ potential pairwise boundaries ($z_i = z_j$). The geometry shifts from a simple arrangement to a competitive partition, effectively decoupling the number of regions from the computational cost of a single expert.
\end{remark}

For $k > 1$, the geometry undergoes a fundamental shift from convexity to combinatorics. The router selects a coalition of $k$ experts, corresponding to the singular locus of the $k$-th elementary symmetric tropical polynomial.

\begin{theorem}[Top-k MoE Geometry: Tropical Grassmannians and Order-$k$ Voronoi Diagrams]
\label{thm:topk_geometry}
Let $\binom{\mathcal{N}}{k}$ denote the collection of all possible expert coalitions of size $k$. For any coalition $I \in \binom{\mathcal{N}}{k}$, define its aggregate score as $S_I(\mathbf{x}) = \sum_{i \in I} z_i(\mathbf{x})$.
The decomposition of the input space $\mathbb{R}^{d_{in}}$ induced by a Top-$k$ router constitutes an \textbf{Affine Order-$k$ Voronoi Diagram}. Specifically, the decision boundaries form the singular locus of the tropical polynomial $S_{\text{topk}}(\mathbf{x})$. Geometrically, this partition arises from the intersection of the input space (embedded via $\mathbf{W}_r$) with the \textbf{Normal Fan} of the $(k, N)$-Hypersimplex $\Delta_{k, N}$:
\begin{equation*}
    \mathcal{B}_k = \left\{ \mathbf{x} \in \mathbb{R}^{d_{in}} \;\middle|\; \exists I, J \in \binom{\mathcal{N}}{k}, I \neq J : S_I(\mathbf{x}) = S_J(\mathbf{x}) = S_{\text{topk}}(\mathbf{x}) \right\}.
\end{equation*}
Crucially, while the activation region $\Omega_I$ for a specific coalition $I$ is a convex polyhedron, the aggregate region served by any single expert is a \textbf{non-convex} union of such polyhedra.
\end{theorem}

\begin{figure}
    \centering
    \includegraphics[width=\textwidth]{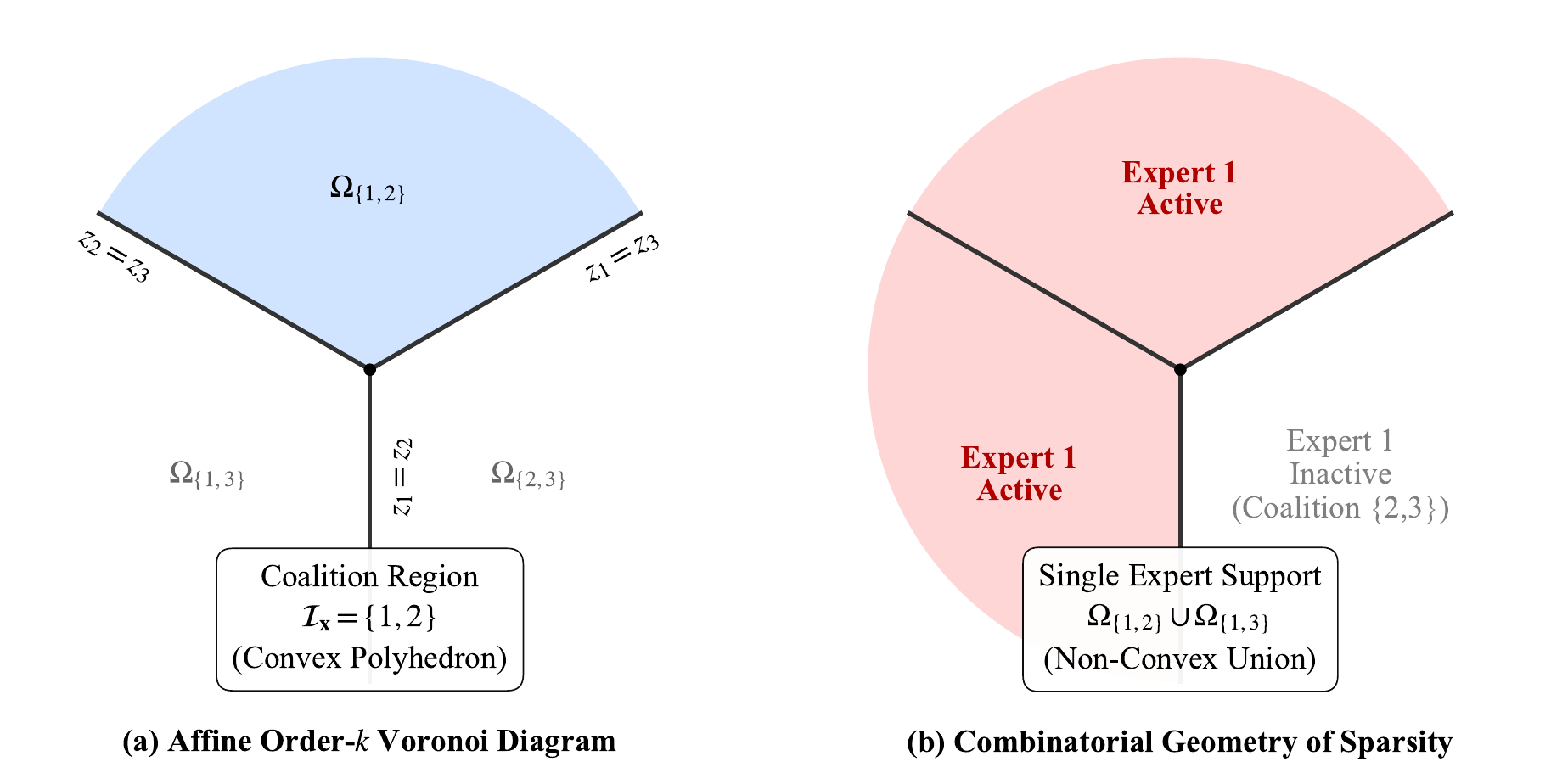}
    \caption{\textbf{Geometric Interpretation of Top-$k$ Routing.} Visualized for an MoE layer with $N=3$ experts and $k=2$ active selections. \textbf{(a)} The activation region for a specific expert coalition (e.g., $\Omega_{\{1,2\}}$) is a \textbf{convex} polyhedral cone, defined by the intersection of half-spaces where $z_1 > z_3$ and $z_2 > z_3$. \textbf{(b)} The total support region for a single expert (e.g., Expert 1) spans multiple adjacent cones ($\Omega_{\{1,2\}} \cup \Omega_{\{1,3\}}$). This union forms a fundamentally \textbf{non-convex} space, demonstrating the complex topological fragmentation processed by individual sparse experts.}
    \label{fig:topk_geometry}
    \vspace{-0.4cm}
\end{figure}

\begin{example}[Top-2 Routing with 3 Experts] 
To build intuition for the geometric shift induced by conditional routing, consider a minimal MoE layer with $N=3$ experts and $k=2$ active selections. The router generates three affine logits $z_1(\mathbf{x}), z_2(\mathbf{x}), z_3(\mathbf{x})$. There are $\binom{3}{2} = 3$ possible expert coalitions: $\{1, 2\}$, $\{1, 3\}$, and $\{2, 3\}$. In Figure~\ref{fig:topk_geometry} (a), the coalition $I = \{1, 2\}$ is selected if and only if its logits dominate the unselected expert:
\begin{equation*}
    \Omega_{\{1,2\}} = \{ \mathbf{x} \in \mathbb{R}^{d_{in}} \mid z_1(\mathbf{x}) > z_3(\mathbf{x}) \text{ and } z_2(\mathbf{x}) > z_3(\mathbf{x}) \}.
\end{equation*}
Geometrically, this region is the intersection of two affine half-spaces, forming a \textbf{convex} polyhedral cone. The boundaries separating these coalitions constitute the Affine Order-2 Voronoi Diagram. However, the geometric complexity emerges when analyzing the network from the perspective of a single expert. As shown in Figure~\ref{fig:topk_geometry} (b), Expert 1 is active whenever the router selects either $\{1, 2\}$ or $\{1, 3\}$. Its total effective domain is the union of these two adjacent cones: $\Omega_{\{1,2\}} \cup \Omega_{\{1,3\}}$. Because this combined region wraps around the origin, it is \textbf{non-convex}.
\end{example}

\begin{remark}[Geometric Intuition and Exactness]
\label{rem:geometry_and_exactness}
    Unlike Top-1 routing, which merely tiles the space, Top-$k$ routing projects data onto the dual of the Hypersimplex $\Delta_{k,N}$. Its combinatorial complexity stems from the Normal Fan of $\Delta_{k,N}$, closely related to the \textbf{Tropical Grassmannian} $\text{Gr}(k, N)$ \citep{speyer2003tropical} (proof in Appendix \ref{app:proof_topk}). Importantly, this tropical partition is \textbf{exact} even with Softmax-gated experts: while Softmax smooths outputs within regions, the decision boundaries, where the active set $\mathcal{I}_{\mathbf{x}}$ changes, depend only on logits' rank ordering. Since Softmax is strictly monotonic, it preserves this order, making the Tropical Hypersurface the precise topological skeleton of the router.
\end{remark}

\subsection{Capacity Analysis: The Combinatorial Slicing Theorem}
\label{subsec:capacity_analysis}

In this subsection, we quantify the \textbf{Geometric Capacity}, defined as the maximum number of linear regions the model induces in the input space $\mathbb{R}^{d_{in}}$. To facilitate rigorous counting, we rely on the theory of hyperplane arrangements. The maximal capacity is achieved only under specific geometric conditions.

\begin{definition}[General Position]
\label{def:general_position}
A set of hyperplanes $\mathcal{A} = \{\mathcal{H}_1, \dots, \mathcal{H}_n\}$ in an ambient space $\mathbb{R}^d$ is in \textbf{General Position} \citep{cover1965geometrical} if:
(i) No two hyperplanes are parallel.
(ii) For any $1 \le k \le d$, the intersection of any $k$ hyperplanes is a subspace of dimension $d-k$.
(iii) The intersection of any $d+1$ hyperplanes is empty.
Under random initialization of weights, this condition holds with probability 1 \citep{stanley2007introduction,orlik2013arrangements,hanin2019complexity}.
\end{definition}

\begin{figure}
    \centering
    \includegraphics[width=\textwidth]{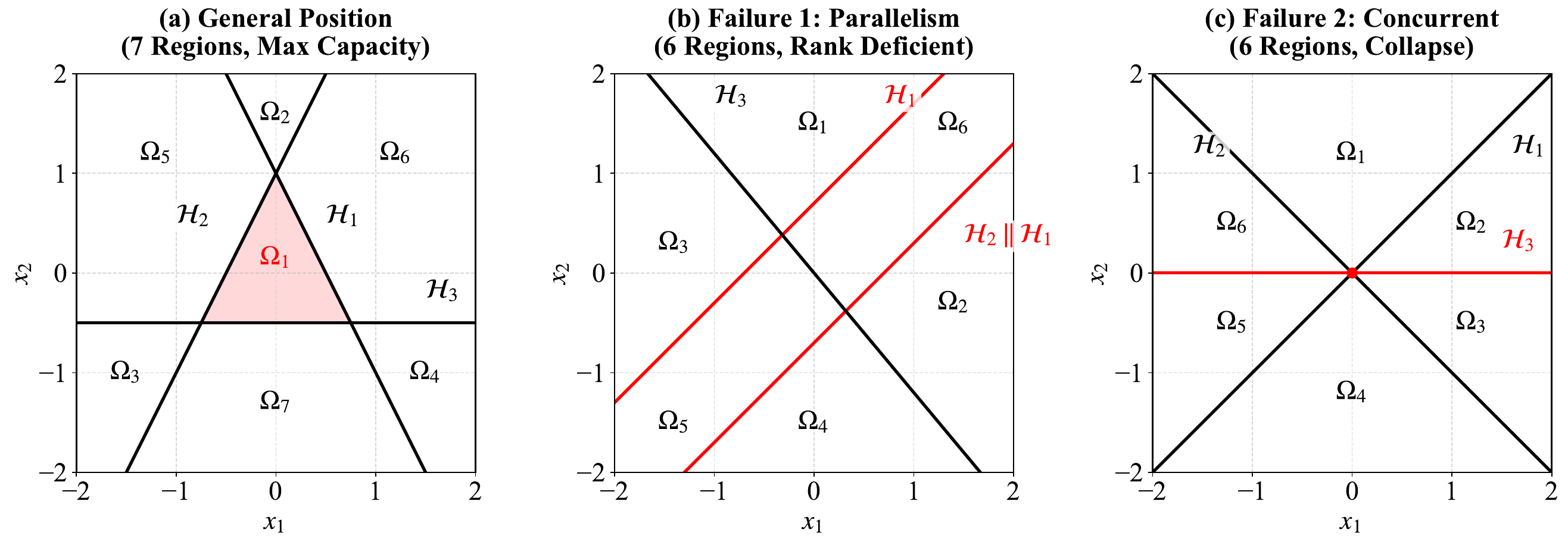}
    \caption{\textbf{Geometric impact of General Position on Capacity.} \textbf{(a)} In general position ($d_{in}=2, n=3$), hyperplanes form a closed cell (red shadow, $\Omega_1$), maximizing the number of regions ($\Phi(3,2)=7$). \textbf{(b \& c)} Degenerate arrangements fail to enclose this maximal cell structure, reducing the combinatorial capacity. In MoE routers, these degeneracies cause a direct collapse of the Tropical Hypersurface.}
    \label{fig:general_position}
    \vspace{-0.4cm}
\end{figure}

\begin{example}[General Position vs. Degenerate Arrangements in $\mathbb{R}^2$]
\label{ex:general_position_failure}
To concretize Definition \ref{def:general_position} and illustrate the geometric conditions required for maximal expressivity, consider a minimal ambient space $\mathbb{R}^{d_{in}}$ with $d_{in}=2$ and $n=3$ affine hyperplanes, representing three ReLU neurons or routing boundaries. 
\begin{itemize}
    \item \textbf{General Position (Maximal Capacity):} As shown in Figure \ref{fig:general_position} (a), the three lines are mutually non-parallel and do not intersect at a single point. This arrangement achieves the theoretical maximum number of linear regions according to Zaslavsky's theorem in Definition \ref{def:zaslavsky}: $\Phi(3, 2) = \binom{3}{0} + \binom{3}{1} + \binom{3}{2} = 1 + 3 + 3 = 7$ regions ($\Omega_1 \dots \Omega_7$).
    
    \item \textbf{Failure Case 1: Parallelism (Rank Deficiency):} In Figure \ref{fig:general_position} (b), two hyperplanes are parallel ($\mathcal{H}_1 \parallel \mathcal{H}_2$). This violates Condition (i) of Definition \ref{def:general_position}. The hyperplanes fail to enclose a central bounded region, causing the highest-order combinatorial term to vanish, reducing the total regions to 6. 
    
    \item \textbf{Failure Case 2: Concurrent Intersection (Dimensional Collapse):} In Figure \ref{fig:general_position} (c), all three lines intersect at the origin. This violates Condition (iii) (intersection of $d_{in}+1=3$ hyperplanes is non-empty). The central polyhedral cell ($\Omega_1$) collapses to a dimension of zero, strictly reducing the capacity to 6 regions. Later in Section 5, we will see that this specific geometric degeneracy manifests as \textit{Angular Collapse} (Proposition \ref{prop:angular_collapse}), where routing boundaries fail to transversally partition the data manifold.
\end{itemize}
\end{example}

Consider a single dense ReLU layer with width $H$. The geometry is governed by an arrangement of $H$ hyperplanes defined by the weight vectors. While the capacity bound for this setting is established in the literature \citep{pascanu2013number,montufar2014number,serra2018bounding,xiong2020number,xiong2024number}, we revisit it through the lens of \textbf{Geometric Duality}. As illustrated in Figure~\ref{fig:primal_dual} (Appendix~\ref{app:proof_dense}), the linear regions in the input space (Primal) are intimately linked to the vertices of a Newton Zonotope in the weight space (Dual). This perspective allows us to unify the analysis of dense and sparse models under a single geometric framework.

\begin{proposition}[Dense Base Capacity Upper Bound]
\label{prop:dense_capacity}
For a dense layer with width $H$ and input dimension $d_{in}$, the number of linear regions $N_{\text{dense}}$ is upper bounded by the cumulative binomial sum, which determines the asymptotic scaling:
\begin{equation*}
    N_{\text{dense}} \le \Phi(H, d_{in}) = \sum_{j=0}^{d_{in}} \binom{H}{j} = \Theta(H^{d_{in}}).
\end{equation*}
Equality in the first step holds if and only if the weights are in general position (Definition~\ref{def:general_position}). The complexity is dominated by the highest-order term $H^{d_{in}}$ scaled by the structural constant $1/d_{in}!$.
\end{proposition}

This bound confirms the rigid compute-capacity coupling: increasing geometric complexity requires scaling $H$, which linearly increases parameters ($H \cdot d_{in}$) and computation.

\begin{remark}[Novelty: Geometric Duality via Zonotopes]
    While the polynomial bound $\mathcal{O}(H^{d_{in}})$ is established, our proof in Appendix \ref{app:proof_dense} introduces a perspective via \textbf{Legendre-Fenchel Duality}. The linear regions in the primal space are structurally related to the vertices of a \textbf{Newton Zonotope} in the dual augmented weight space $\mathbb{R}^{d_{in}+1}$. Although the mapping is not strictly bijective (the affine arrangement corresponds to a slice of the central arrangement associated with the zonotope), we prove that they share the same asymptotic complexity class $\Theta(H^{d_{in}})$. This unifies the geometric analysis: dense capacity is governed by Zonotope vertex complexity.
\end{remark}

\begin{corollary}[Dense Base Capacity Lower Bound and Tightness]
\label{cor:dense_lower_bound}
There exists a choice of weights for a dense ReLU layer of width $H$ and input dimension $d_{in}$ such that the number of linear regions satisfies
\begin{equation*}
    N_{\text{dense}} \ge \sum_{j=0}^{d_{in}} \binom{H}{j} = \Theta(H^{d_{in}}).
\end{equation*}
Consequently, the upper bound in Proposition~\ref{prop:dense_capacity} is tight up to constant factors.
\end{corollary}

This lower bound is achieved when the hyperplanes defined by the weight vectors are in general position, a condition that holds almost surely under random initialization \citep{montufar2014number,serra2018bounding}. As a result, the geometric capacity of dense networks is polynomial in the width $H$.

\begin{remark}[Polynomial Nature of Dense Capacity]
    Corollary~\ref{cor:dense_lower_bound} establishes that dense architectures admit no super-polynomial growth in geometric capacity with respect to width.
    This polynomial scaling will serve as a tight baseline when contrasting dense networks with sparsely routed architectures in subsequent sections.
\end{remark}

Consider a Top-1 MoE layer with $N$ experts, each of width $H$. According to Proposition \ref{prop:top1_geometry}, the router partitions the input space into $N$ convex Voronoi cells $\Omega_1, \dots, \Omega_N$.

\begin{theorem}[Top-1 Capacity Upper Bound]
\label{thm:top1_capacity}
Let $N_{\text{MoE-1}}$ denote the number of linear regions induced by the MoE layer. Assuming the hyperplane arrangements of the experts are in general position (Definition~\ref{def:general_position}) within their respective active Voronoi cells $\Omega_i$, the total capacity is upper bounded by the sum of individual expert capacities, which scales linearly with the number of experts $N$ in the asymptotic regime ($H \gg d_{in}$):
\begin{equation*}
    N_{\text{MoE-1}} \le N \cdot \Phi(H, d_{in}) = \Theta(N \cdot H^{d_{in}}).
\end{equation*}
\end{theorem}

This theorem confirms that Top-1 MoE decouples geometric capacity from inference cost. By introducing $N$ experts, the model expands its \textit{geometric width} linearly, achieving a capacity of $N \cdot H^{d_{in}}$ while maintaining the active FLOPs of a single expert (proportional to $H$). The detailed proof of Theorem~\ref{thm:top1_capacity} is provided in Appendix \ref{app:proof_top1}.

\begin{remark}[Boundary Effects and the Wide-Expert Regime]
    A rigorous counting of linear regions must also account for the router's decision boundaries $\mathcal{H}^r$, which may transversally intersect the expert hyperplanes $\mathcal{H}^e$ within a cell. The exact upper bound for the arrangement restricted to a cell is $\Phi(H + |\mathcal{H}^r_{loc}|, d_{in})$, where $|\mathcal{H}^r_{loc}|$ is the number of router hyperplanes bounding the cell. However, in the \textbf{Wide-Expert Regime} ($H \gg N$), the expert width dominates the routing complexity ($H \gg |\mathcal{H}^r_{loc}|$). Therefore, the contribution of router cuts to the region count becomes negligible, rendering the approximation $\sum \Phi(H, d_{in})$ asymptotically exact.
\end{remark}

\begin{corollary}[Top-1 MoE Capacity Lower Bound and Tightness]
\label{cor:top1_lower_bound}
In the wide-expert regime, there exists a choice of routing and expert weights such that
the number of linear regions induced by a Top-1 MoE layer satisfies
\begin{equation*}
    N_{\text{MoE-1}} = \Omega(N \cdot H^{d_{in}}).
\end{equation*}
Consequently, the upper bound in Theorem~\ref{thm:top1_capacity} is tight up to constant factors.
\end{corollary}

This confirms that Top-1 routing yields only a linear expansion of geometric capacity relative to dense networks. 

For Top-$k$ routing ($k > 1$), the router selects a coalition of experts, partitioning the input space into $\binom{N}{k}$ distinct routing regions (Theorem \ref{thm:topk_geometry}). Within each region, the output is governed by the sum of $k$ active experts, which induces a local arrangement of $k \times H$ hyperplanes.

\begin{theorem}[Top-$k$ MoE Capacity Upper Bound: Combinatorial Slicing Theorem]
\label{thm:combinatorial_slicing}
Consider a Top-$k$ MoE layer with $N$ experts of width $H$ and input dimension $d_{in}$. Assuming that the union of all expert activation boundaries and routing decision boundaries is in general position (Definition~\ref{def:general_position}), the total number of linear regions $N_{\text{MoE-k}}$ is upper bounded by the sum of capacities across all expert combinations. In the asymptotic regime where $kH \gg d_{in}$, the capacity scales as:
\begin{equation*}
    N_{\text{MoE-k}} \le \binom{N}{k} \cdot \Phi(kH, d_{in}) = \Theta\left(\binom{N}{k} (kH)^{d_{in}}\right).
\end{equation*}
This bound identifies the combinatorial multiplier $\binom{N}{k}$ that drives MoE expressivity by partitioning the space into the Normal Fan of a Hypersimplex.
\end{theorem}

This theorem provides the rigorous basis for the proposition that \textbf{Sparsity is Combinatorial Depth}. The combinatorial term $\binom{N}{k}$ quantifies the topological gain derived from the router's selection mechanism, while the polynomial term $\Phi(kH, d_{in})$ represents the algebraic capacity of the active expert coalition. The detailed proof of Theorem~\ref{thm:combinatorial_slicing} is provided in Appendix \ref{app:proof_topk_capacity}.

\begin{lemma}[Combinatorial Reachability]
\label{lem:topk_lower_routing}
There exists a choice of routing weights such that for every subset $S \subset \mathcal{N}$ with $|S| = k$, the corresponding routing region $R_S = \{\mathbf{x} \in \mathbb{R}^{d_{in}} : \operatorname{Top\text{-}k}(\mathbf{W_r} \mathbf{x}) = S\}$ is non-empty and has positive measure.
\end{lemma}

We now show that the upper bound in Theorem~\ref{thm:combinatorial_slicing} is tight up to constant factors.

\begin{theorem}[Top-$k$ MoE Capacity Lower Bound]
\label{thm:topk_capacity_lower}
Consider a Top-$k$ MoE layer with $N$ experts of width $H$ and input dimension $d_{in}$.
There exists a choice of routing and expert weights such that the number of linear regions satisfies
\begin{equation*}
    N_{\text{MoE-k}} = \Omega\!\left(\binom{N}{k} \cdot H^{d_{in}}\right).
\end{equation*}
\end{theorem}

\begin{table}[t]
    \centering
    \caption{\textbf{Comparison of Geometric Capacity and Inference Cost.} 
    We contrast the three architectures across physical constraints (Parameters) and theoretical expressivity (Capacity).
    $d_{\text{in}}$: input dimension, $H$: expert width, $N$: total experts, $k$: active experts.
    The function $\Phi(n, d) = \sum_{j=0}^{d} \binom{n}{j}$ denotes Zaslavsky's bound.
    \textbf{Active Parameters} represents the memory bandwidth required for a single token's forward pass.
    The final column reports a normalized view of Top-$k$ MoE, where the expert activation scale is aligned with dense models.}
    \label{tab:capacity_comparison}
    \renewcommand{\arraystretch}{1.5} 
    \resizebox{\textwidth}{!}{
    \begin{tabular}{ccccc}
        \toprule
        \textbf{Metric} 
        & \textbf{Dense Network} 
        & \textbf{MoE (Top-1)} 
        & \textbf{MoE (Top-$k$)} 
        & \textbf{MoE (Top-$k$, Normalized)} \\ 
        \midrule
        \textbf{Active Parameters (Inference Cost)} 
            & $H d_{\text{in}}$ 
            & $H d_{\text{in}} + N d_{\text{in}}$ 
            & $k H d_{\text{in}} + N d_{\text{in}}$
            & $H d_{\text{in}} + N d_{\text{in}}$ \\
        \midrule
        \textbf{Total Parameters (Memory Cost)} 
            & $H d_{\text{in}}$ 
            & $N H d_{\text{in}} + N d_{\text{in}}$ 
            & $N H d_{\text{in}} + N d_{\text{in}}$
            & $N H d_{\text{in}} + N d_{\text{in}}$ \\
        \midrule
        \textbf{Exact Capacity (Tight Bound)} 
            & $\Phi(H, d_{\text{in}})$ 
            & $N \cdot \Phi(H, d_{\text{in}})$ 
            & $\binom{N}{k} \cdot \Phi(kH, d_{\text{in}})$
            & $\binom{N}{k} \cdot \Phi(H, d_{\text{in}})$ \\
        \midrule
        \textbf{Asymptotic ($H \gg d_{\text{in}}$)} 
            & $\Theta(H^{d_{\text{in}}})$ 
            & $\Theta(N \cdot H^{d_{\text{in}}})$ 
            & $\Theta\!\left(\binom{N}{k} (kH)^{d_{\text{in}}}\right)$
            & $\Theta\!\left(\binom{N}{k} H^{d_{\text{in}}}\right)$ \\
        \bottomrule
    \end{tabular}
    }
\end{table}

\begin{remark}[Tightness and Interpretation]
    Combining Theorems~\ref{thm:combinatorial_slicing} and~\ref{thm:topk_capacity_lower},
    we conclude that the geometric capacity of Top-$k$ MoE architectures scales as $\Theta\!\left(\binom{N}{k} (kH)^{d_{in}}\right)$ up to constant factors. In contrast to dense and Top-1 architectures, where capacity grows only polynomially with width, Top-$k$ routing introduces a genuine combinatorial depth via expert selection. The detailed proof of Theorem~\ref{thm:topk_capacity_lower} is provided in Appendix \ref{app:proof_topk_capacity_low}.
\end{remark}

In summary, our analysis establishes that sparsity acts as a geometric multiplier, fundamentally decoupling a model's expressivity from its computational burden. Table \ref{tab:capacity_comparison} quantitatively contrasts the scaling regimes, revealing that while dense networks are restricted to polynomial growth relative to their width, Top-k MoE architectures achieve a combinatorial expansion in the number of linear regions. This demonstrates that sparse routing enables the construction of highly complex topologies without a commensurate increase in active floating-point operations.

\begin{example}[Iso-Active-Compute Comparison]
To ground the theoretical bounds in Table~\ref{tab:capacity_comparison}, we compare the exact geometric capacity of three architectures under a strict \textbf{iso-active-compute constraint}. We fix the input dimension $d_{\text{in}} = 5$ and restrict each model to activate exactly $20$ ReLU neurons per input token. We allocate a total parameter pool of $N=16$ experts for the MoE models. By Definition \ref{def:zaslavsky}, we calculate the maximal number of linear regions:
\begin{itemize}
    \item \textbf{Dense Network:} A single layer of width $H = 20$. 
    \begin{equation*}
        N_{\text{Dense}} = \Phi(20, 5) = \mathbf{21{,}700} \text{ regions.}
    \end{equation*}
    
    \item \textbf{Top-1 MoE:} The router selects $k=1$ expert from $N=16$. To maintain the compute constraint, the expert width is $H = 20$.
    \begin{equation*}
        N_{\text{Top-1}} = 16 \cdot \Phi(20, 5) = 16 \times 21{,}700 = \mathbf{347{,}200} \text{ regions.}
    \end{equation*}
    
    \item \textbf{Top-4 MoE:} The router selects $k=4$ experts from $N=16$. To strictly match the 20 active neurons, each expert width is reduced to $H = 5$ (so $kH = 20$).
    \begin{equation*}
        N_{\text{Top-4}} = \binom{16}{4} \cdot \Phi(20, 5) = 1820 \times 21{,}700 = \mathbf{39{,}494{,}000} \text{ regions.}
    \end{equation*}
\end{itemize}
While all three models execute identical dense matrix multiplications per token ($20 \times d_{\text{in}}$ active weights), Top-1 MoE achieves a linear ($16\times$) expansion in capacity by tiling the space. In contrast, Top-4 MoE achieves an $1820\times$ capacity multiplier over the dense baseline. This demonstrates that under fixed inference budgets, splitting compute into fine-grained routed experts unlocks an exponential combinatorial expansion in the function space.
\end{example}

\subsection{Effective Capacity under the Manifold Hypothesis}
\label{subsec:effective_capacity}

The combinatorial bounds derived in Theorem \ref{thm:combinatorial_slicing} represent the \textit{theoretical ceiling} assuming the data spans the entire ambient space $\mathbb{R}^{d_{in}}$. However, the \textit{Manifold Hypothesis} posits that real-world high-dimensional data concentrates near a low-dimensional manifold $\mathcal{M} \subset \mathbb{R}^{d_{in}}$ with intrinsic dimension $d_{eff} \ll d_{in}$ \citep{bengio2013representation,fefferman2016testing,ansuini2019intrinsic,pope2021intrinsic}.

Prior work on dense networks reveals a large gap between theoretical capacity and what is realized in practice. \citet{hanin2019deep} showed that although the number of linear regions grows exponentially with depth or polynomially with input dimension, the number of active regions intersecting the data manifold scales only \textbf{linearly} with the total number of neurons. This suggests that geometric complexity is limited by the intrinsic dimension of the data. We capture this using \textbf{Effective Capacity}, the count of linear regions intersecting $\mathcal{M}$. While a dense layer partitions the ambient space $\mathbb{R}^{d_{in}}$, partitioning the manifold requires a transversality condition to prevent degenerate alignment of hyperplanes.

\begin{assumption}[Almost-Sure Transversality]
\label{ass:transversality}
We assume the weights of the expert layers (including bias terms) are drawn i.i.d. from a continuous distribution (e.g., Gaussian or Uniform) on $\mathbb{R}^{H \times d_{in}}$. The induced \textbf{affine} hyperplane arrangement $\mathcal{A}$ is almost surely transversal to the data manifold $\mathcal{M}$ \citep{thom1954quelques,abraham1963transversality,mirsky1971transversal,goodman1993geometric}. Specifically, with probability 1, for any subset of $j \le d_{eff}$ hyperplanes, their intersection with $\mathcal{M}$ is either empty or a submanifold of codimension $j$.
\end{assumption}

\begin{remark}[Validity and Limits of the Transversality Assumption]
While Assumption \ref{ass:transversality} poses a strong geometric condition, it serves as the standard and necessary analytical vehicle for deriving the \textbf{expressivity ceiling} (i.e., maximum capacity upper bounds) of neural networks. We clarify its role and limitations in two aspects:
\begin{itemize}
    \item \textbf{Theoretical Necessity for Upper Bounds:} As demonstrated in Example \ref{ex:general_position_failure}, any violation of transversality (e.g., hyperplanes becoming tangent to $\mathcal{M}$ or intersecting in degenerate subspaces) strictly \textit{reduces} the combinatorial count of linear regions. Therefore, to quantify the maximum potential expressivity that the architecture \textit{can} achieve, assuming transversal intersection is mathematically imperative.
    
    \item \textbf{Initialization vs. Post-Training Dynamics:} By Sard's Theorem \citep{sard1942measure}, this assumption holds with probability 1 at random initialization due to the absolute continuity of the weight distribution. We acknowledge that post-training, gradient descent induces weight correlations, meaning hyperplanes may align with the manifold's curvature, violating strict i.i.d. transversality. However, established empirical studies \citep{hanin2019deep, serra2018bounding} indicate that while training reduces the \textit{absolute number} of linear regions compared to initialization, the \textit{asymptotic polynomial growth rates} governing the effective capacity on $\mathcal{M}$ remain structurally bounded by the transversal case. 
\end{itemize}
Consequently, Assumption \ref{ass:transversality} should be interpreted as defining the \textit{optimal geometric state} required for the MoE architecture to realize its full combinatorial depth.
\end{remark}

Assumption~\ref{ass:transversality} implicitly requires that the number of affine hyperplanes acting on the manifold is at least its intrinsic dimension; otherwise, transversality is impossible for dimensional reasons.

\begin{proposition}[Upper Bound on Effective Dense Capacity]
\label{prop:dense_effective}
Let $\mathcal{M} \subset \mathbb{R}^{d_{in}}$ be a smooth submanifold of intrinsic dimension $d_{eff}$. Under Assumption \ref{ass:transversality}, the number of linear regions of a dense layer with width $H$ that have non-empty intersection with $\mathcal{M}$ satisfies:
\begin{equation*}
    N_{\text{dense}}^{\text{eff}} \le \Phi(H, d_{eff}) = \sum_{j=0}^{d_{eff}} \binom{H}{j} = \mathcal{O}(H^{d_{eff}}).
\end{equation*}
\end{proposition}

\begin{remark}[Tightness for Dense Effective Capacity]
    Under Assumption~\ref{ass:transversality}, generic hyperplane arrangements restricted to $\mathcal{M}$ attain the polynomial scaling $\Theta(H^{d_{eff}})$ up to constant factors.
    This behavior is consistent with existing results on the intrinsic expressivity of dense ReLU networks \citep{hanin2019deep}.
\end{remark}

We now investigate a pivotal question: \textbf{\textit{does the sparsity gain survive this dimensional collapse?}} While the data manifold $\mathcal{M}$ lies in a lower-dimensional subspace, the Top-$k$ routing partition creates cones in the full ambient space $\mathbb{R}^{d_{in}}$. To rigorously quantify the intersection, we analyze the expected capacity using symmetry and volume arguments inspired by spherical integral geometry.

\begin{theorem}[Upper Bound on Expected Effective Capacity]
\label{thm:effective_topk}
Let $\mathcal{M} \subset \mathbb{R}^{d_{in}} \setminus \{0\}$ be a smooth compact submanifold with intrinsic dimension $d_{eff}$. Let $\pi(\mathcal{M}) = \{ \mathbf{x}/\|\mathbf{x}\| \mid \mathbf{x} \in \mathcal{M} \}$ be its projection onto the unit sphere $\mathbb{S}^{d_{in}-1}$, and let $\text{Vol}(\cdot)$ denote the volume measure on the sphere. Assume the router weights $\mathbf{W}_r \in \mathbb{R}^{N \times d_{in}}$ are drawn from an isotropic distribution (e.g., $\mathcal{N}(0, I)$), and the expert weights satisfy Assumption \ref{ass:transversality}. We further assume that, restricted to each routing cone, the induced expert hyperplane arrangement is full-rank with respect to $\mathcal{M}$ whenever $kH \ge d_{eff}$. The expected number of linear regions intersecting $\mathcal{M}$ admits the following upper bound:
\begin{equation*}
    \mathbb{E}_{\mathbf{W}_r}\left[N_{\text{MoE-k}}^{\text{eff}}\right] \le \binom{N}{k} \cdot \frac{\text{Vol}(\pi(\mathcal{M}))}{\text{Vol}(\mathbb{S}^{d_{in}-1})} \cdot \Phi(kH, d_{eff}) + \mathcal{O}\left( (kH)^{d_{eff}-1} \right).
\end{equation*}
In the asymptotic regime where $kH \gg d_{eff}$, this bound scales as:
\begin{equation*}
    \mathbb{E}[N_{\text{MoE-k}}^{\text{eff}}] = \mathcal{O}\left( \frac{\text{Vol}(\pi(\mathcal{M}))}{\text{Vol}(\mathbb{S}^{d_{in}-1})} \cdot \binom{N}{k} (kH)^{d_{eff}} \right).
\end{equation*}
\end{theorem}

The upper bound in Theorem~\ref{thm:effective_topk} captures the maximal effective capacity under the manifold hypothesis (proof in Appendix \ref{app:proof_effective_capacity}). A key question is whether the combinatorial factor $\binom{N}{k}$ is achievable on the data manifold or reduced by intrinsic dimensionality. We show that, under the same isotropy and transversality assumptions, Top-$k$ MoE architectures attain a matching lower bound in expectation, and the combinatorial gain from sparse routing persists on $\mathcal{M}$.

\begin{theorem}[Lower Bound on Expected Effective Capacity]
\label{thm:effective_topk_lower}
Let $\mathcal{M} \subset \mathbb{R}^{d_{in}} \setminus \{0\}$ be a smooth compact submanifold of intrinsic dimension $d_{eff}$ such that $\mathrm{Vol}(\pi(\mathcal{M})) > 0$. Assume the router weights $\mathbf{W}_r \in \mathbb{R}^{N \times d_{in}}$ are drawn i.i.d.\ from an isotropic distribution and that the expert weights satisfy Assumption~\ref{ass:transversality}. If $kH \ge d_{eff}$, then there exists a constant $c_{\mathcal{M}} > 0$, depending only on $\mathcal{M}$, such that the expected number of linear regions intersecting $\mathcal{M}$ satisfies
\begin{equation*}
    \mathbb{E}_{\mathbf{W}_r} \!\left[N_{\mathrm{MoE}\text{-}k}^{\mathrm{eff}}\right] \ge c_{\mathcal{M}} \cdot \frac{\mathrm{Vol}(\pi(\mathcal{M}))}{\mathrm{Vol}(\mathbb{S}^{d_{in}-1})} \cdot \binom{N}{k} \cdot (kH)^{d_{eff}}.
\end{equation*}
\end{theorem}

\begin{remark}[Tightness of Effective Capacity Scaling]
Theorems~\ref{thm:effective_topk} and~\ref{thm:effective_topk_lower} together imply that, under the manifold hypothesis, the effective geometric capacity of Top-$k$ MoE architectures satisfies
\begin{equation*}
    \mathbb{E}\!\left[N_{\mathrm{MoE}\text{-}k}^{\mathrm{eff}}\right] = \Theta\!\left(
    \frac{\mathrm{Vol}(\pi(\mathcal{M}))}{\mathrm{Vol}(\mathbb{S}^{d_{in}-1})} \cdot \binom{N}{k} \cdot (kH)^{d_{eff}} \right),
\end{equation*}
demonstrating that the combinatorial gain induced by sparse routing persists despite intrinsic dimensional collapse. The detailed proof of Theorem~\ref{thm:effective_topk_lower} is provided in Appendix \ref{app:proof_effective_topk_lower}.
\end{remark}

\begin{corollary}[Combinatorial Resilience: Ratio of Upper Bounds]
\label{cor:resilience}
Consider a Dense network and a Top-$k$ MoE network with equivalent expert width $H$. Let $B_{\text{dense}}$ and $B_{\text{MoE}}$ denote the theoretical upper bounds on their effective capacities derived in Proposition \ref{prop:dense_effective} and Theorem \ref{thm:effective_topk}, respectively.
Assuming the manifold has non-vanishing spherical measure (i.e., $\text{Vol}(\pi(\mathcal{M})) = \Theta(1)$) and $kH \gg d_{eff}$, the ratio of these \textbf{upper bounds} scales asymptotically as:
\begin{equation*}
    \frac{B_{\text{MoE}}}{B_{\text{dense}}} = \mathcal{O}\!\left(\binom{N}{k} k^{d_{eff}}\right).
\end{equation*}
\end{corollary}

\textbf{Interpretation.} This result provides a rigorous geometric foundation for the \textbf{Combinatorial Resilience} of MoE architectures. While the upper bound for Dense Models suffers from \textit{Geometric Collapse} (strictly governed by $H^{d_{eff}}$), the bound for MoE models maintains a multiplicative factor $\binom{N}{k}$. This combinatorial multiplier arises from the spherical partitioning of the router and remains structurally independent of the local manifold dimension $d_{eff}$, suggesting that sparse routing preserves a form of depth that is robust to dimension collapse.

\begin{remark}[The Rank Condition and Granularity Limit]
\label{rem:rank_condition}
    The validity of the polynomial expansion in Theorem \ref{thm:effective_topk} relies on the assumption that the local arrangement is full-rank with respect to the manifold, i.e., $kH \ge d_{eff}$. If $kH < d_{eff}$, the arrangement becomes \textbf{Rank Deficient}. Geometrically, this implies the expert hyperplanes fail to form bounded cells on $\mathcal{M}$, causing the higher-order terms in Zaslavsky’s sum to vanish. This imposes a theoretical lower bound on expert width to ensure geometric non-degeneracy.
\end{remark}

\section{Architectural Laws: Deriving Optimal Sparsity}
\label{sec:architectural_laws}

The capacity analysis in Section \ref{sec:capacity} establishes that the expressivity of MoE is dominated by the combinatorial term $\binom{N}{k}$. In this section, we translate these geometric insights into prescriptive architectural laws. We formulate the architecture design as a constrained optimization problem: \textbf{\textit{How to allocate a fixed budget of parameters and FLOPs to maximize the effective geometric capacity?}} This derivation rigorously justifies the design evolution of State-of-the-Art architectures, specifically the shift towards \textbf{Fine-Grained Experts} (e.g., DeepSeek-MoE) and the necessity of \textbf{Shared Experts}.

\subsection{The Law of Optimal Sparsity}
\label{subsec:optimal_sparsity}

We address the core architectural design question: \textit{\textbf{Given a fixed parameter and compute budget, should we use a few massive experts (Coarse-grained) or many tiny experts (Fine-grained)?}} 

To answer this, we analyze the scaling of the geometric capacity upper bound under a \textbf{Granularity Transformation} $\mathcal{T}_m$. Consider a baseline architecture $\mathcal{A}_{base}$ with $N$ experts, active set size $k$, and expert width $H$. Following common MoE scaling practices (e.g., as seen in architectures like DeepSeek-MoE \citep{dai2024deepseekmoe}), we fix the embedding dimension $d_{model}$ and split the FFN intermediate width. The transformation $\mathcal{A}_{fine} = \mathcal{T}_m(\mathcal{A}_{base})$ with splitting factor $m \in \mathbb{N}_{\ge 1}$ is defined by:

\begin{itemize}
    \item \textbf{Split Experts:} $N' = m N$.
    \item \textbf{Shrink Width:} $H' = H / m$. (Ensures \textbf{Iso-Memory}: $N'H' = NH$).
    \item \textbf{Scale Active Set:} $k' = m k$. (Ensures \textbf{Iso-FLOPs}: $k'H' = kH$).
\end{itemize}

Under this rigorous Iso-Resource constraint, we derive the scaling law for the theoretical upper bound of the geometric capacity.

\begin{theorem}[Asymptotics of the Combinatorial Capacity Upper Bound]
\label{thm:fine_grained}
Let $\mathcal{B}(N, H, k) := \binom{N}{k} H^{d_{eff}}$ denote the principal term of the effective capacity upper bound derived in Theorem \ref{thm:effective_topk}, where we suppress multiplicative constants depending solely on $k$ and $d_{eff}$.

Consider the \textbf{Sparse Regime} where $N \to \infty$, while the granularity $m$, base active size $k$, and effective dimension $d_{eff}$ remain fixed constants independent of $N$. Assuming the architecture satisfies the necessary dimensional conditions for transversality (specifically $H/m \ge d_{eff}$), the ratio of the capacity upper bounds $\mathcal{G}_{ub}(m)$ under the transformation $\mathcal{T}_m$ admits the following asymptotic expansion:
\begin{equation*}
    \mathcal{G}_{ub}(m) := \frac{\mathcal{B}(mN, H/m, mk)}{\mathcal{B}(N, H, k)} = \left( \frac{Ne}{k} \right)^{(m-1)k} m^{-d_{eff} - \frac{1}{2}} (1 + o(1)).
\end{equation*}
Consequently, the upper bound on capacity exhibits polynomial growth in $N$ of degree $(m-1)k$, which increases linearly with the granularity factor $m$.
\end{theorem}

\begin{proof}
    Please refer to Appendix \ref{app:proof_fine_grained} for the detailed proof.
\end{proof}

\textbf{Interpretation: The Trade-off.} As a consequence of the asymptotic expansion above, taking the logarithm of the leading term reveals the structural trade-off between combinatorial expansion and width contraction:
\begin{equation*}
    \ln \mathcal{G}_{ub}(m) \sim \underbrace{(m-1)k \ln \left( \frac{Ne}{k} \right)}_{\text{Combinatorial Gain (Linear in } m)} - \underbrace{\left(d_{eff} + \frac{1}{2}\right) \ln m}_{\text{Width Penalty (Logarithmic in } m)}.
\end{equation*}
\textbf{Geometric Conclusion:} The combinatorial upper bound scales linearly with $m$ (in the exponent) while the width penalty scales only logarithmically. This suggests that fine-grained splitting provides a potential for a rapid expansion of the combinatorial upper bound on the number of linear regions, in the sense of Theorem~\ref{thm:fine_grained}. This gain is driven by the expansion of the routing configuration space (which we heuristically describe, by analogy, as the Tropical Grassmannian), provided the local experts remain sufficiently wide to cover the manifold dimension.

\begin{remark}[Escaping the Curse of Dimensionality]
    A naive application of the bound using the input dimension $d_{in}$ (e.g., 4096) would imply that the width penalty term $m^{-d_{in}}$ is catastrophic. However, under the \textbf{Manifold Hypothesis} ($d_{eff} \ll d_{in}$) \citep{bengio2013representation,fefferman2016testing}, the penalty remains small. Empirical studies consistently estimate the effective dimension to be $d_{eff} \approx 10 \sim 20$ \citep{ansuini2019intrinsic,pope2021intrinsic}. This ensures that the architecture operates in a regime where Combinatorial Depth theoretically dominates Physical Width.
\end{remark}

Based on Theorem \ref{thm:fine_grained}, if splitting experts yields super-linear gains, \textbf{\textit{why not let $m \to \infty$?}} However, the validity of this asymptotic analysis is constrained by the geometric rank of the experts.

\begin{corollary}[Necessary Rank Condition for the Validity of the Scaling Law]
\label{cor:granularity_limit}
The asymptotic scaling law derived in Theorem \ref{thm:fine_grained} relies on the validity of Assumption \ref{ass:transversality}. A necessary condition for this assumption to hold at the expert level is that the number of affine hyperplanes per expert is at least the intrinsic dimension of the data manifold, i.e.,
We define the critical granularity limit as the integer part of this ratio:
\begin{equation*}
    m_{max} := \left\lfloor \frac{H}{d_{eff}} \right\rfloor.
\end{equation*}
To prevent geometric collapse, the granularity parameter $m$ must satisfy the constraint:
\begin{equation*}
    m \le m_{max}.
\end{equation*}
If this condition is violated, the hyperplane arrangement becomes rank-deficient with respect to the manifold. Consequently, the leading-order term of the capacity upper bound (of order $H^{d_{eff}}$) can no longer be attained, rendering the polynomial scaling law derived in Theorem~\ref{thm:fine_grained} inapplicable.
\end{corollary}

\begin{remark}[Geometric Interpretation: Newton Polytope Collapse]
    In the tropical framework, this limit corresponds to the degeneration of the \textbf{Newton Polytope} associated with the network's tropical polynomial \citep{zhang2018tropical}. When $H/m < d_{eff}$, the expert weight vectors (with probability one under isotropic initialization) fail to span the tangent space of the data manifold. This results in \textbf{Tropical Rank Deficiency}, where the Newton Polytope flattens onto a lower-dimensional face. This phenomenon can be interpreted as a geometric manifestation of the representation collapse observed in excessive sparsity regimes \citep{chi2022representation}.
\end{remark}

In summary, Theorem \ref{thm:fine_grained} and Corollary \ref{cor:granularity_limit} suggest that one should optimally set to maximize capacity. However, this conclusion assumes an \textit{ideal} router capable of realizing any combinatorially valid partition. In practice, the routing mechanism is constrained by the geometry of dot products. Even if $m \le m_{max}$, the router may fail to distribute data across the available experts, a phenomenon we analyze next.

\subsection{The Geometric Limit: Angular Collapse}
\label{subsec:angular_collapse}

If splitting experts yields super-linear gains, \textit{\textbf{what prevents the router from fully utilizing this capacity?}} The limit is imposed by the \textbf{geometric stability} of the routing mechanism. Recall from Theorem \ref{thm:topk_geometry} that the decision boundaries form the Normal Fan of a Hypersimplex. This structure partitions the input space $\mathbb{R}^{d_{in}}$ based on angular directions, effectively operating on the unit sphere $\mathbb{S}^{d_{in}-1}$. We identify a geometric failure mode where the effective capacity vanishes not due to a lack of parameters, but due to a misalignment between the data distribution and the projective geometry of the router. We term this phenomenon \textit{Angular Collapse} (Figure \ref{fig:angular_collapse}-(a)). We emphasize that this is a geometric phenomenon independent of learning dynamics.

\begin{definition}[Tropical Saturation]
Let $S_{\text{topk}}(\mathbf{x})$ be the tropical polynomial defining the routing score, and let $\mathcal{T}(S_{\text{topk}}) \subset \mathbb{R}^{d_{in}}$ be its Tropical Hypersurface (the singular locus). A router is in a state of $\epsilon$-Tropical Saturation with respect to a data manifold $\mathcal{M}$ if:
\begin{equation*}
    \inf_{\mathbf{x} \in \mathcal{M}} \text{dist}(\mathbf{x}, \mathcal{T}(S_{\text{topk}})) > \epsilon.
\end{equation*}
In this state, the data lies strictly within the interior of the linearity regions (cones), and the local routing decisions become locally constant.
\end{definition}

To quantify the risk of saturation, we analyze the angular footprint of the data. When the data is uncentered (i.e., the mean magnitude is large relative to the variation), its projection onto the sphere concentrates into a vanishingly small region. The following proposition formalizes this geometric bottleneck.

\begin{proposition}[Angular Collapse of Data Support]
\label{prop:angular_collapse}
Let the input dimension $d_{in} \ge 2$ be fixed. Let the data input be of the form $\mathbf{x} = \boldsymbol{\mu} + \boldsymbol{\varepsilon}$, where $\boldsymbol{\mu} \in \mathbb{R}^{d_{in}}$ is a deterministic vector with $\boldsymbol{\mu} \neq \mathbf{0}$, and $\boldsymbol{\varepsilon}$ is a random noise vector satisfying $\|\boldsymbol{\varepsilon}\| \le R$ almost surely for some finite constant $R > 0$.

Define the projected data support on the unit sphere as the image of the ball $B(\boldsymbol{\mu}, R)$ under normalization:
\begin{equation*}
    \mathcal{C} := \left\{ \frac{\mathbf{x}}{\|\mathbf{x}\|} \;\middle|\; \mathbf{x} \in B(\boldsymbol{\mu}, R) \right\} \subset \mathbb{S}^{d_{in}-1}.
\end{equation*}
For $\|\boldsymbol{\mu}\| > R$, this set corresponds to a spherical cap. Let $A(\mathcal{C})$ denote the normalized surface area of $\mathcal{C}$ under the uniform measure on $\mathbb{S}^{d_{in}-1}$.

Then, in the uncentered regime where $\|\boldsymbol{\mu}\| \to \infty$ (specifically for $\|\boldsymbol{\mu}\| \gg R$), we have:
\begin{equation*}
    A(\mathcal{C}) = C(d_{in}) \cdot \left( \frac{R}{\|\boldsymbol{\mu}\|} \right)^{d_{in}-1} \, (1 + o(1)),
\end{equation*}
where $C(d_{in}) \in (0, \infty)$ is a constant depending only on the dimension.
\end{proposition}

\begin{proof}
    Please refer to Appendix \ref{app:proof_angular_collapse} for a detailed proof.
\end{proof}

\begin{remark}[Geometric Bottleneck for Sparse Routing]
Proposition \ref{prop:angular_collapse} establishes a geometric bottleneck caused by uncentered data. Consider a random Top-$k$ routing mechanism where the router weights are isotropically distributed. The expected number of routing coalitions whose activation cones have non-negligible overlap with the data support $\mathcal{C}$ is bounded by the cap's area. Specifically, up to dimension-dependent constants, the effective capacity admits the following upper bound:
\begin{equation}
    \mathbb{E}[N_{active}] = O\left( A(\mathcal{C}) \right) = O\left( \|\boldsymbol{\mu}\|^{-(d_{in}-1)} \right).
\end{equation}
This suggests that as the data mean shifts away from the origin, the combinatorial diversity available to the model collapses geometrically.
\end{remark}

This result elucidates the geometric instability of fine-grained sparsity. Since the router's partition is inherently projective (slicing directions rather than positions), fine-grained splitting ($N \to \infty$) results in increasingly narrow cones. Proposition \ref{prop:angular_collapse} implies that without geometric regularization (centering), the \textit{solid angle} of the data decays polynomially with the distance from the origin. This renders the vast majority of the combinatorial capacity $\binom{N}{k}$ structurally unreachable, necessitating the affine rectification strategy discussed in the subsequent section.

\begin{figure}
\centering

\begin{tikzpicture}[scale=0.95, >=Stealth]

\tikzstyle{axis} = [->, black!60, thin]
\tikzstyle{vector} = [->, thick, purple!80, shorten >=1pt]
\tikzstyle{router} = [red!80!black, thick] 
\tikzstyle{router_dashed} = [red!80!black, dashed, thick]
\tikzstyle{cone_active} = [fill=red!10, fill opacity=0.9]
\tikzstyle{dead_zone} = [fill=gray!10] 
\tikzstyle{data_blob} = [draw=blue!80, thick, fill=blue!10, rotate=30]
\tikzstyle{res_blob} = [draw=teal!60!black, thick, fill=teal!10]
\tikzstyle{label_box} = [
  fill=white,
  inner sep=4pt,
  align=center,
  font=\footnotesize,
  draw=black!20,
  rounded corners=3pt,
  drop shadow={opacity=0.3, shadow xshift=1pt, shadow yshift=-1pt}
]

\begin{scope}[local bounding box=leftpanel]

    \begin{scope}
        \clip (-0.5,-0.5) rectangle (5.0, 5.0);
        \fill[dead_zone] (0,0) rectangle (5,5);
        \fill[white] (0,0) -- (42:6) arc (42:53:6) -- cycle;
        \fill[cone_active] (0,0) -- (42:6) arc (42:53:6) -- cycle;
    \end{scope}

    \draw[axis] (-0.5,0) -- (4.8,0) node[right] {\small $x_1$};
    \draw[axis] (0,-0.5) -- (0,4.8) node[above] {\small $x_2$};
    \node[below left, font=\footnotesize, black!70] at (0,0) {$\mathbf{0}$};

    \begin{scope}[shift={(3.5, 3.0)}] 
        \draw[data_blob] (0,0) ellipse (1.2 and 0.4); 
        \node[blue!80!black, font=\footnotesize] at (0.4, 0.9) {$\mathcal{M}$ (Manifold)};
    \end{scope}

    \draw[vector] (0,0) -- (3.5, 3.0) node[midway, below right, font=\footnotesize] {$\boldsymbol{\mu} \approx \mathbb{E}[\mathbf{x}]$};

    \draw[router] (0,0) -- (42:5.0) node[right, font=\scriptsize, yshift=2pt] {$\mathbf{w}_i^r$};
    \draw[router_dashed] (0,0) -- (47.5:5.0); 
    \draw[router] (0,0) -- (53:5.0) node[above, font=\scriptsize, xshift=-4pt] {$\mathbf{w}_j^r$};

    \node[label_box, draw=red!40] (collapse) at (1.2, 3.5) {
        \textbf{Tropical Saturation}\\
        \scriptsize Locally Constant Routing\\
        \scriptsize ($\mathbf{x} \mapsto \mathcal{I}_{\mathbf{x}} = \text{const.}$)
    };

    \draw[->, red!60, thick] (collapse.south) to[out=-90, in=150] (2.2, 2.5);

    \node[gray!60, font=\scriptsize, align=center] at (3.5, 0.5) {\textit{Combinatorial}\\\textit{Dead Zone}};
    \node[gray!60, font=\scriptsize, align=center] at (1.0, 4.5) {\textit{Combinatorial}\\\textit{Dead Zone}};

    \node[anchor=north] at (2.4, -0.9) {\small \textbf{(a) Uncentered Data: Angular Collapse}};

\end{scope}

\begin{scope}[xshift=9.5cm, local bounding box=rightpanel]

    \fill[teal!10, opacity=0.8] (0,0) -- (30:3.5) arc (30:150:3.5) -- cycle;
    \fill[orange!10, opacity=0.8] (0,0) -- (150:3.5) arc (150:270:3.5) -- cycle;
    \fill[blue!05, opacity=0.8] (0,0) -- (270:3.5) arc (270:390:3.5) -- cycle;

    \draw[axis] (-3.5,0) -- (3.5,0);
    \draw[axis] (0,-3.5) -- (0,3.5);
    \fill[black!70] (0,0) circle (1.5pt);

    \draw[res_blob] plot [smooth cycle, tension=0.8] coordinates {
        (0.9,0.5) (0.2,1.0) (-0.8,0.6)
        (-1.0,-0.5) (0.3,-0.9) (1.1,-0.3)
    };
    \node[teal!60!black, font=\footnotesize] at (0,0) {$\mathcal{M}_{\text{res}}$};

    \draw[router_dashed] (0,0) -- (30:3.8);
    \draw[router_dashed] (0,0) -- (150:3.8);
    \draw[router_dashed] (0,0) -- (270:3.8);

    \draw[->, thick, gray!60, dashed] 
        (-4.0, 2.8) to[bend left=25] 
        node[midway, above, font=\scriptsize, align=center, yshift=1pt] {Carrier Isolation\\$\mathbf{x} \mapsto \mathbf{x} - E_s(\mathbf{x})$} 
        (-1.5, 1.5);

    \node[label_box, draw=teal!60] at (2.5, -2.8) {
        \textbf{Transversality Restored}\\
        \scriptsize Satisfying Assumption \ref{ass:transversality}\\
        \scriptsize Max Capacity $\binom{N}{k}$
    };

    \node[anchor=north] at (0, -0.9) {\small \textbf{(b) Shared Expert: Carrier Isolation}};

\end{scope}

\end{tikzpicture}

\caption{\textbf{Geometric Regularization via Affine Carrier Isolation.} 
\textbf{(a) Angular Collapse:} When the input distribution is uncentered ($\|\boldsymbol{\mu}\| \gg \text{Var}(\mathbf{x})$), the manifold $\mathcal{M}$ subtends a vanishingly small solid angle. This leads to \textbf{Tropical Saturation}, a failure mode where the manifold is contained within the interior of a single Weyl chamber, rendering routing decisions locally constant.
\textbf{(b) Affine-Residual Reparameterization:} The Shared Expert $E_s$ isolates the global affine carrier. This effectively centers the residual manifold $\mathcal{M}_{\text{res}}$ at the origin, \textbf{removing the geometric obstruction} to Assumption~\ref{ass:transversality}. This ensures the architecture realizes its full combinatorial capacity $\binom{N}{k}$ by allowing decision boundaries to intersect the high-density regions of the data.}
\label{fig:angular_collapse}
\end{figure}

\subsection{Functional Centering via Affine-Residual Decomposition}
\label{sec:shared_experts}

To mitigate the geometric degeneracy (Angular Collapse) identified in Section \ref{subsec:angular_collapse}, we mathematically formalize the role of the \textit{Shared Expert} mechanism, recently popularized by State-of-the-Art architectures \citep{dai2024deepseekmoe}. We first provide its formal structural definition.

\begin{definition}[Shared Expert Mechanism]
\label{def:shared_expert}
In a MoE layer with $N$ routed experts, a shared expert $E_s: \mathbb{R}^{d_{in}} \to \mathbb{R}^{d_{out}}$ is an auxiliary dense feed-forward network that is decoupled from the Top-$k$ routing mechanism $\mathcal{G}$. It is uniformly activated for all inputs $\mathbf{x} \in \mathcal{M}$, with a constant routing weight of $1$.
\end{definition}

Mathematically, adding a shared expert is a straightforward additive adjustment. However, it serves a critical role in decoupling functional optimization from the geometric partition of the router. We first formalize the failure mode. When the target function possesses a dominant mean component, routed experts are forced to reconstruct this global bias. In the hard-gating limit, this pressure encourages the router to saturate, selecting a fixed subset of experts solely for their magnitude contribution. This leads to a \textit{Trivial Partition}, where the active set $\mathcal{I}_{\mathbf{x}}$ becomes locally constant across the data support, effectively collapsing the MoE layer to a single affine map. The shared expert provides an explicit reparameterization to prevent this functional collapse.

\begin{proposition}[Functional Isolation of the Affine Carrier]
\label{prop:affine_reparam}
Let $F: \mathbb{R}^{d_{in}} \to \mathbb{R}^{d_{out}}$ be the function represented by the MoE layer. The inclusion of a shared expert $E_s(\mathbf{x}) = \mathbf{W}_s \mathbf{x} + \mathbf{b}_s$ decomposes the layer output as:
\begin{equation*}
    F(\mathbf{x}) = \underbrace{\mathcal{A}(\mathbf{x})}_{\text{Global Affine Carrier}} + \underbrace{\mathcal{R}(\mathbf{x}; \mathcal{G})}_{\text{Routed Residual}},
\end{equation*}
where $\mathcal{A}(\mathbf{x}) \equiv E_s(\mathbf{x})$ is independent of the routing policy, and $\mathcal{R}(\mathbf{x}; \mathcal{G}) = \sum_{i \in \mathcal{I}_{\mathbf{x}}} G_i(\mathbf{x}) E_i(\mathbf{x})$ captures the sparsity-dependent non-linearity. 
\end{proposition}

While this decomposition is algebraically trivial, it prevents a specific optimization failure mode. When a target function possesses a dominant global mean, forcing sparsely routed experts to approximate this bias incentivizes the router to select a fixed subset of experts constantly to maintain a stable output magnitude. This causes the active set $\mathcal{I}_{\mathbf{x}}$ to become locally constant across the data support, rendering the combinatorial capacity $\binom{N}{k}$ unutilized.

\begin{remark}[Geometric Centering and Transversality]
    \label{rem:restoring_transversality}
    Input-space normalization (e.g., LayerNorm) centers the features $\mathbf{x}$, but does not address the bias in the target function. The shared expert provides an explicit reparameterization to absorb this target bias. By isolating the global carrier into $\mathcal{A}(\mathbf{x})$, the sparsely routed experts are only tasked with fitting zero-mean residuals. Geometrically, this effectively centers the residual data manifold relative to the router's projective cones (Figure \ref{fig:angular_collapse} (b)), satisfying the transversality requirement (Assumption \ref{ass:transversality}) and ensuring the decision boundaries actively partition the data support.
\end{remark}

In summary, the shared expert is a structural necessity for maintaining the accessibility of the combinatorial capacity $\binom{N}{k}$. By isolating the global affine carrier from the routing dynamics, it prevents the sparse routing from being hijacked by the global bias. This ensures that the decision boundaries of the tropical hypersurface intersect the high-variance regions of the data, preserving the combinatorial depth that would otherwise be lost to functional saturation.

\section{Conclusion}

In this work, we present the first analysis of MoE derived from the perspective of tropical geometry. We establish an algebraic isomorphism between the Top-$k$ routing mechanism and the Normal Fan of a Hypersimplex, revealing that \textbf{Sparsity is Combinatorial Depth} which scales geometric capacity by $\binom{N}{k}$. Crucially, we prove that this combinatorial structure grants MoE architectures a topological resilience against the capacity collapse observed in dense networks on low-dimensional manifolds. Furthermore, we translate these geometric insights into prescriptive design rules, deriving the asymptotic capacity bounds of fine-grained experts and establishing the shared expert mechanism as a structural necessity to avert geometric routing collapse. Our framework provides a novel theoretical foundation for understanding why sparse conditional computation is expressively efficient.

\bibliographystyle{plainnat} 
\bibliography{reference}

\appendix

\section{Missing Proof of Theorem \ref{thm:topk_geometry}}
\label{app:proof_topk}

\begin{proof}
Let the affine score (logit) for the $i$-th expert be denoted by $z_i(\mathbf{x}) = (\mathbf{w}^r_i)^\top \mathbf{x} + b^r_i$. The Top-k router selects the coalition $I \in \binom{\mathcal{N}}{k}$ that maximizes the aggregate score. The activation region (cell) for a specific coalition $I$ is defined as:
\begin{equation*}
    \mathcal{V}(I) = \left\{ \mathbf{x} \in \mathbb{R}^d \;\middle|\; \sum_{p \in I} z_p(\mathbf{x}) \ge \sum_{q \in J} z_q(\mathbf{x}), \quad \forall J \in \binom{\mathcal{N}}{k} \setminus \{I\} \right\}.
\end{equation*}
To determine the minimal set of inequalities defining $\mathcal{V}(I)$, consider an arbitrary competing coalition $J \neq I$. Let $K = I \cap J$ be the set of common experts. We partition the index sets as disjoint unions:
\begin{equation*}
    I = K \cup (I \setminus J), \quad J = K \cup (J \setminus I).
\end{equation*}
Since $|I| = |J| = k$, the symmetric differences must have equal cardinality: $|I \setminus J| = |J \setminus I| = m$, where $1 \le m \le k$.
We expand the dominance inequality $S_I(\mathbf{x}) \ge S_J(\mathbf{x})$:
\begin{equation*}
    \begin{aligned}
        \sum_{p \in I} z_p(\mathbf{x}) &\ge \sum_{q \in J} z_q(\mathbf{x}) \\
        \underbrace{\sum_{p \in K} z_p(\mathbf{x})}_{\text{Common Terms}} + \sum_{u \in I \setminus J} z_u(\mathbf{x}) &\ge \underbrace{\sum_{q \in K} z_q(\mathbf{x})}_{\text{Common Terms}} + \sum_{v \in J \setminus I} z_v(\mathbf{x}).
    \end{aligned}
\end{equation*}
Subtracting the common terms $\sum_{p \in K} z_p(\mathbf{x})$ from both sides, the inequality reduces to a comparison of the distinct experts:
\begin{equation}
    \label{eq:reduced_ineq}
    \sum_{u \in I \setminus J} z_u(\mathbf{x}) \ge \sum_{v \in J \setminus I} z_v(\mathbf{x}).
\end{equation}
\textbf{Redundancy Analysis via Transitivity:} We claim that any constraint where $m > 1$ is redundant.
Suppose $m=2$. Let $I \setminus J = \{u_1, u_2\}$ and $J \setminus I = \{v_1, v_2\}$. The inequality Eq.~\eqref{eq:reduced_ineq} becomes:
\begin{equation*}
    z_{u_1}(\mathbf{x}) + z_{u_2}(\mathbf{x}) \ge z_{v_1}(\mathbf{x}) + z_{v_2}(\mathbf{x}).
\end{equation*}
However, if $\mathbf{x} \in \mathcal{V}(I)$, then coalition $I$ must also defeat the intermediate coalitions $L_1 = (I \setminus \{u_1\}) \cup \{v_1\}$ and $L_2 = (I \setminus \{u_2\}) \cup \{v_2\}$. Note that $L_1, L_2 \in \binom{\mathcal{N}}{k}$ are valid coalitions that differ from $I$ by only $m=1$ element.
The constraints imposed by $L_1$ and $L_2$ are:
\begin{equation*}
    \begin{cases}
        S_I(\mathbf{x}) \ge S_{L_1}(\mathbf{x}) \iff z_{u_1}(\mathbf{x}) \ge z_{v_1}(\mathbf{x}), \\
        S_I(\mathbf{x}) \ge S_{L_2}(\mathbf{x}) \iff z_{u_2}(\mathbf{x}) \ge z_{v_2}(\mathbf{x}).
    \end{cases}
\end{equation*}
Summing these two fundamental inequalities yields:
\begin{equation*}
    (z_{u_1}(\mathbf{x}) \ge z_{v_1}(\mathbf{x})) + (z_{u_2}(\mathbf{x}) \ge z_{v_2}(\mathbf{x})) \implies z_{u_1}(\mathbf{x}) + z_{u_2}(\mathbf{x}) \ge z_{v_1}(\mathbf{x}) + z_{v_2}(\mathbf{x}).
\end{equation*}
By induction on $m$, all inequalities with $m > 1$ are implied by the set of inequalities with $m=1$.
Thus, the facets of $\mathcal{V}(I)$ are strictly determined by the neighbors where $|I \setminus J| = 1$. Let $J = (I \setminus \{u\}) \cup \{v\}$ for any $u \in I$ and $v \notin I$. The minimal representation is:
\begin{equation*}
    \mathcal{V}(I) = \bigcap_{u \in I} \bigcap_{v \in \mathcal{N} \setminus I} \left\{ \mathbf{x} \;\middle|\; z_u(\mathbf{x}) \ge z_v(\mathbf{x}) \right\}.
\end{equation*}
Substituting the affine form $z_i(\mathbf{x}) = (\mathbf{w}^r_i)^\top \mathbf{x} + b^r_i$, we obtain the intersection of $k(N-k)$ half-spaces:
\begin{equation*}
    \mathcal{V}(I) = \left\{ \mathbf{x} \in \mathbb{R}^d \;\middle|\; (\mathbf{w}^r_u - \mathbf{w}^r_v)^\top \mathbf{x} + (b^r_u - b^r_v) \ge 0, \quad \forall u \in I, v \notin I \right\}.
\end{equation*}

The selection function corresponds to the $k$-th elementary symmetric tropical polynomial:
\begin{equation*}
    E_k(\mathbf{x}) = \max_{I \in \binom{\mathcal{N}}{k}} \left( \sum_{i \in I} (\mathbf{w}^r_i)^\top \mathbf{x} + b^r_i \right) 
    = \max_{I \in \binom{\mathcal{N}}{k}} \left( \left( \sum_{i \in I} \mathbf{w}^r_i \right)^\top \mathbf{x} + \sum_{i \in I} b^r_i \right).
\end{equation*}
This formulation matches the definition of a convex support function $h_{\mathcal{P}}(\mathbf{x}) = \max_{\mathbf{v} \in \mathcal{V}} (\mathbf{v}^\top \mathbf{x} + c_{\mathbf{v}})$. The vertices generating the dual polytope are:
\begin{equation*}
    \mathcal{V}_{\text{dual}} = \left\{ \mathbf{v}_I = \sum_{i \in I} \mathbf{w}^r_i \;\middle|\; I \in \binom{\mathcal{N}}{k} \right\}.
\end{equation*}
Let $\mathbf{W}_r \in \mathbb{R}^{N \times d}$ be the matrix with rows $\mathbf{w}^r_i$. Let $\Delta_{k,N} = \text{conv}\{\mathbf{e}_I \mid \|\mathbf{e}_I\|_1 = k\} \subset \mathbb{R}^N$ be the standard $(k,N)$-Hypersimplex. We observe that each $\mathbf{v}_I$ is a linear projection of a vertex of $\Delta_{k,N}$:
\begin{equation*}
    \mathbf{v}_I = (\mathbf{W}_r)^\top \mathbf{e}_I.
\end{equation*}
From convex geometry, the domains of linearity of a support function $h_{\mathcal{P}}$ coincide with the cones of the Normal Fan of the polytope $\mathcal{P} = \text{conv}(\mathcal{V}_{\text{dual}})$.
Crucially, the combinatorial adjacency derived in Part 1 ($m=1$) corresponds exactly to the edge structure of the Hypersimplex: vertices $\mathbf{e}_I$ and $\mathbf{e}_J$ are connected by an edge in $\Delta_{k,N}$ if and only if $|I \Delta J| = 2$ (i.e., one swap).
Therefore, assuming $\mathbf{W}_r$ is in general position (preserving the vertex hull), the partition $\mathcal{P} = \{\mathcal{V}(I)\}_I$ is isomorphic to the Normal Fan of the projected Hypersimplex $(\mathbf{W}_r)^\top \Delta_{k,N}$.
\end{proof}

\section{Missing Proof of Proposition \ref{prop:dense_capacity}}
\label{app:proof_dense}

\begin{proof}
We establish the capacity bound by analyzing the geometric structure induced by the layer's weights. While the polynomial bound for dense networks is a classical result \citep{montufar2014number,serra2018bounding,xiong2020number,xiong2024number}, we provide a derivation via \textbf{Legendre-Fenchel Duality}, connecting the primal linear regions to the vertices of a Newton Polytope.

First, we formalize the geometric setting. Let the layer be parameterized by the augmented weight matrix $\tilde{\mathbf{W}} = [\mathbf{W}, \mathbf{b}] \in \mathbb{R}^{H \times (d_{in}+1)}$ \citep{stanley2007introduction}. We assume the hyperplanes defined by the rows of $\tilde{\mathbf{W}}$ are in \textit{General Position} (Definition \ref{def:general_position}). This is a standard assumption in the literature, as the set of singular parameters (where hyperplanes are parallel or intersect in a degenerate subspace) has Lebesgue measure zero in the parameter space \citep{glorot2010understanding,montufar2014number}. Under random initialization with continuous distributions, this property holds with probability 1 \citep{serra2018bounding,hinz2019framework}.

Consider a dense layer with $H$ neurons and input dimension $d_{in}$. The decision boundary of the $i$-th ReLU neuron is defined by the affine hyperplane:
\begin{equation*}
    \mathcal{H}_i = \{ \mathbf{x} \in \mathbb{R}^{d_{in}} \mid \mathbf{w}_i^\top \mathbf{x} + b_i = 0 \}, \quad \text{for } i = 1, \dots, H.
\end{equation*}
The collection of these hyperplanes forms an arrangement $\mathcal{A} = \{\mathcal{H}_1, \dots, \mathcal{H}_H\}$. The number of linear regions $N_{\text{dense}}$ corresponds to the number of connected cells in the complement space $\mathbb{R}^{d_{in}} \setminus \bigcup_{i=1}^H \mathcal{H}_i$.

According to Zaslavsky's Theorem \citep{zaslavsky1975facing}, the number of regions in an arrangement of $H$ hyperplanes in general position in $d_{in}$-dimensional space is given exactly by the partial sum of binomial coefficients:
\begin{equation*}
    N_{\text{dense}} = \Phi(H, d_{in}) = \sum_{j=0}^{d_{in}} \binom{H}{j}.
\end{equation*}
We aim to derive the scaling law of $N_{\text{dense}}$ with respect to $H$. We expand the summation explicitly:
\begin{equation*}
    N_{\text{dense}} = \binom{H}{d_{in}} + \binom{H}{d_{in}-1} + \dots + \binom{H}{1} + \binom{H}{0}.
\end{equation*}
Since we consider the regime where $H \to \infty$ while $d_{in}$ is fixed ($H \gg d_{in}$), the sum is dominated by the term with the highest index $j=d_{in}$. We analyze the magnitude of the $j$-th binomial coefficient using its polynomial expansion:
\begin{align*}
    \binom{H}{j} &= \frac{H!}{j!(H-j)!} = \frac{1}{j!} \prod_{i=0}^{j-1} (H-i) \\
    &= \frac{1}{j!} \left( H^j - \frac{j(j-1)}{2}H^{j-1} + \mathcal{O}(H^{j-2}) \right) \\
    &\sim \frac{1}{j!} H^j.
\end{align*}
Substituting this polynomial expansion into the sum for $N_{\text{dense}}$:
\begin{align*}
    N_{\text{dense}} &= \binom{H}{d_{in}} + \sum_{j=0}^{d_{in}-1} \binom{H}{j} \\
    &= \left[ \frac{1}{d_{in}!} H^{d_{in}} + \mathcal{O}(H^{d_{in}-1}) \right] + \sum_{j=0}^{d_{in}-1} \mathcal{O}(H^j) \\
    &= \frac{1}{d_{in}!} H^{d_{in}} + \mathcal{O}(H^{d_{in}-1}).
\end{align*}
Since the summation term $\sum_{j=0}^{d_{in}-1} \mathcal{O}(H^j)$ is strictly bounded by $d_{in} \times \mathcal{O}(H^{d_{in}-1})$, the total complexity is strictly governed by the highest-order term.

To rigorously connect this to the \textbf{Dual Geometry}, we consider the standard lifting of the arrangement to $\mathbb{R}^{d_{in}+1}$ via the augmented weights $\tilde{\mathbf{w}}_i=[\mathbf{w};b_i]$. The number of regions in the affine arrangement in $\mathbb{R}^{d_{in}}$ scales asymptotically the same as the number of vertices $V(\mathcal{P})$ of the Zonotope generated by the vectors $\tilde{\mathbf{W}}$ in the dual space $\mathbb{R}^{d_{in}+1}$:
\begin{equation*}
    V(\mathcal{P}) = 2 \sum_{k=0}^{d_{in}} \binom{H-1}{j} \sim 2 \cdot \frac{1}{d_{in}!} H^{d_{in}} = \Theta(H^{d_{in}}).
\end{equation*}
Both the primal (Arrangement) and dual (Zonotope) perspectives yield the same polynomial scaling law. Since $d_{in}$ is a structural constant independent of width, we conclude:
\begin{equation*}
    N_{\text{dense}}  = \Theta(H^{d_{in}}).
\end{equation*}
This completes the proof.

\begin{figure}
    \centering
    \begin{tikzpicture}[
        scale=0.9,
        >=stealth,
        axis/.style={->, gray!50, thin},
        line_style/.style={thick, blue!80!black},
        boundary_style/.style={very thick, blue!80!black}, 
        weight_vec/.style={->, thick, red!80!black},
        zono_fill/.style={fill=red!10, draw=red!80!black, thick},
        region_highlight/.style={fill=blue!10},
        vertex_highlight/.style={fill=red, circle, inner sep=1.5pt},
        map_line/.style={->, dashed, thick, purple, shorten >=2pt, shorten <=2pt}
    ]

    \begin{scope}[local bounding box=primal_scope]
        \node[font=\bfseries] at (0, 3.5) {Primal Space ($\mathbb{R}^{d_{in}}$)};

        \draw[axis] (-2.5,0) -- (2.5,0) node[right, gray] {$x_1$};
        \draw[axis] (0,-2.5) -- (0,2.5) node[above, gray] {$x_2$};

        \begin{scope}
            \clip (-2.2, -2.2) rectangle (2.2, 2.2);
            
            \fill[region_highlight] (0.3, 1.2) -- (0.3, 2.5) -- (1.6, 2.5) -- cycle;

            \draw[line_style] (0.3, -2.5) -- (0.3, 2.5); 
            \draw[line_style] (-2.5, 0.3) -- (2.5, 0.3); 
            \draw[line_style] (-2.5, -1.6) -- (1.6, 2.5); 

            \draw[boundary_style] (0.3, 0.3) -- (0.3, 1.2); 
            \draw[boundary_style] (0.3, 0.3) -- (2.2, 0.3); 
            \draw[boundary_style] (0.3, 1.2) -- (1.3, 2.2); 
        \end{scope}

        \node[blue!80!black, right] at (0.3, -1.5) {$H_1$};
        \node[blue!80!black, above] at (-1.5, 0.3) {$H_2$};
        \node[blue!80!black, left] at (-0.5, -0.6) {$H_3$};
        
        \node[blue!80, font=\large] at (0.6, 2.0) {$\mathcal{R}_v$};

        \coordinate (primal_point) at (0.6, 1.9);
    \end{scope}

    \begin{scope}[xshift=6.5cm, local bounding box=dual_scope]
        \node[font=\bfseries] at (0, 3.5) {Dual Space (Weights)};
        \node[font=\small, gray] at (0, 3.1) {Newton Zonotope};

        \draw[axis] (-2.5,0) -- (2.5,0);
        \draw[axis] (0,-2.5) -- (0,2.5);

        \coordinate (O) at (0,0);
        \coordinate (w1) at (1.3, 0);  
        \coordinate (w2) at (0, 1.3);  
        \coordinate (w3) at (-0.8, 0.8);

        \coordinate (v0) at (0,0);
        \coordinate (v1) at (1.3, 0);
        \coordinate (v12) at (1.3, 1.3); 
        \coordinate (v123) at (0.5, 2.1);
        \coordinate (v23) at (-0.8, 2.1);
        \coordinate (v3) at (-0.8, 0.8);

        \draw[zono_fill] (v0) -- (v1) -- (v12) -- (v123) -- (v23) -- (v3) -- cycle;

        \draw[weight_vec] (O) -- (w1) node[below] {$\mathbf{w}_1$};
        \draw[weight_vec] (O) -- (w2) node[right] {$\mathbf{w}_2$};  
        \draw[weight_vec] (O) -- (w3) node[left] {$\mathbf{w}_3$};

        \node[vertex_highlight] (dual_point) at (v12) {};
        \node[red, right, font=\small] at (1.4, 1.3) {$\mathbf{v} = \mathbf{w}_1 + \mathbf{w}_2$};
        
        \node[red!80!black] at (0.5, -1.5) {$\mathcal{Q} = \bigoplus_{i} [\mathbf{0}, \mathbf{w}_i]$};
    \end{scope}

    \draw[map_line] (primal_point) to[bend left=20] node[midway, above, font=\small\bfseries, fill=white, inner sep=1pt] {Duality Mapping $\mathcal{D}_\Phi$} (dual_point);

    \end{tikzpicture}
    
    \caption{\textbf{Geometric Capacity via Duality and Activation Patterns.} (\textbf{Left}) The Primal input space $\mathbb{R}^d$ is partitioned into convex polyhedral cells by a hyperplane arrangement $\mathcal{A} = \{H_1, H_2, H_3\}$, where each hyperplane corresponds to the decision boundary of a ReLU neuron. A specific linear region $\mathcal{R}_{\mathbf{v}}$ is defined by its \textbf{activation pattern} $\mathbf{s} \in \{0, 1\}^3$, representing the binary state of all neurons. For the shaded region $\mathcal{R}_{\mathbf{v}}$, the pattern is $\mathbf{s} = \{1, 1, 0\}$, indicating that neurons 1 and 2 are active ($H_1, H_2$ ``on'') while neuron 3 is inactive. (\textbf{Right}) In the Dual weight space, the layer is represented by a \textbf{Newton Zonotope} $\mathcal{Q}$, which is the Minkowski sum of line segments $\bigoplus_{i=1}^3 [\mathbf{0}, \mathbf{w}_i]$. The \textbf{Duality Mapping} $\mathcal{D}_\Phi$ (Legendre-Fenchel duality) associates primal linear regions $\mathcal{R}_{\mathbf{v}}$ with vertices $\mathbf{v}$ of the zonotope, calculated as the weighted sum of active neurons: $\mathbf{v} = \sum_{i} s_i \mathbf{w}_i$. For the highlighted region, $\mathcal{D}_\Phi(\mathcal{R}_{\mathbf{v}}) = \mathbf{w}_1 + \mathbf{w}_2$. While the exact counts differ between the affine arrangement and the dual zonotope, this framework demonstrates that the total geometric capacity is \textbf{governed by} the vertex complexity of the zonotope, sharing the same asymptotic growth rate $\Theta(H^d)$.}
    \label{fig:primal_dual}
\end{figure}
\end{proof}

\section{Missing Proof of Theorem \ref{thm:top1_capacity}}
\label{app:proof_top1}

\begin{proof}
Let $\mathcal{A} = \{E_1, \dots, E_N\}$ be a Top-1 MoE layer where each expert $E_i$ consists of $H$ ReLU neurons. Let $\mathcal{H}_i = \{h_{i,1}, \dots, h_{i,H}\}$ denote the arrangement of $H$ affine hyperplanes in $\mathbb{R}^{d_{in}}$ corresponding to the activation boundaries of expert $E_i$. We assume that the global collection of hyperplanes, including all expert boundaries $\bigcup_i \mathcal{H}_i$ and the router's decision boundaries $\mathcal{H}^r$, is in general position (Definition~\ref{def:general_position}).

The Top-1 router assigns each input $\mathbf{x} \in \mathbb{R}^{d_{in}}$ to exactly one expert based on the logit scores $z_1(\mathbf{x}), \dots, z_N(\mathbf{x})$. This induces a partition of the input space into $N$ disjoint convex polyhedral cells (Voronoi cells):
\begin{equation*}
    \Omega_i = \left\{ \mathbf{x} \in \mathbb{R}^{d_{in}} \mid z_i(\mathbf{x}) \ge z_j(\mathbf{x}), \, \forall j \neq i \right\}, \quad i=1, \dots, N.
\end{equation*}
The gating function $\mathbf{G}(\mathbf{x})$ is defined as $G_i(\mathbf{x}) = 1$ and $G_j(\mathbf{x}) = 0$ for $j \neq i$ when $\mathbf{x} \in \text{int}(\Omega_i)$.

The global MoE function $F: \mathbb{R}^{d_{in}} \to \mathbb{R}^{d_{out}}$ is defined as $F(\mathbf{x}) = \sum_{j=1}^N G_j(\mathbf{x}) E_j(\mathbf{x})$. Restricting $F$ to the interior of a specific routing cell $\Omega_i$:
\begin{equation*}
    F(\mathbf{x}) \Big|_{\mathbf{x} \in \text{int}(\Omega_i)} = 1 \cdot E_i(\mathbf{x}) + \sum_{j \neq i} 0 \cdot E_j(\mathbf{x}) = E_i(\mathbf{x}).
\end{equation*}
Consequently, the non-differentiable boundaries (singular locus) of $F(\mathbf{x})$ within $\text{int}(\Omega_i)$ are formed exclusively by the hyperplanes in $\mathcal{H}_i$. The hyperplanes of inactive experts $\mathcal{H}_j$ ($j \neq i$) have no effect on the local linear regions within $\Omega_i$.

The total number of linear regions $N_{\text{MoE-1}}$ is the sum of the linear regions formed within each routing cell $\Omega_i$ plus the boundaries between cells. Let $\text{Regions}( {\mathcal{H}_i}|_{\Omega_i} )$ be the number of regions induced by the arrangement $\mathcal{H}_i$ restricted to the convex set $\Omega_i$. By the theory of restricted hyperplane arrangements \citep{zaslavsky1975facing}:
\begin{equation*}
    N_{\text{MoE-1}} = \sum_{i=1}^N \text{Regions}\left( {\mathcal{H}_i}\big|_{\Omega_i} \right).
\end{equation*}
For $H$ hyperplanes in general position, the number of regions inside any convex polyhedral cell $\Omega_i \subseteq \mathbb{R}^{d_{in}}$ is upper bounded by the Zaslavsky function $\Phi(H, d_{in})$:
\begin{equation*}
    \text{Regions}\left( {\mathcal{H}_i}\big|_{\Omega_i} \right) \le \Phi(H, d_{in}) = \sum_{j=0}^{d_{in}} \binom{H}{j}.
\end{equation*}
Summing over all $N$ cells, we obtain the total geometric capacity:
\begin{equation*}
    N_{\text{MoE-1}} \le \sum_{i=1}^N \Phi(H, d_{in}) = N \cdot \Phi(H, d_{in}).
\end{equation*}

In the wide-expert regime where $H \gg d_{in}$, the sum is dominated by the highest-degree term:
\begin{align*}
    \Phi(H, d_{in}) &= \binom{H}{d_{in}} + \binom{H}{d_{in}-1} + \dots + \binom{H}{0} \\
    &= \frac{1}{d_{in}!} \prod_{k=0}^{d_{in}-1} (H-k) + \mathcal{O}(H^{d_{in}-1}) \\
    &= \frac{H^{d_{in}}}{d_{in}!} + \mathcal{O}(H^{d_{in}-1}).
\end{align*}
Substituting this back into the total capacity bound:
\begin{equation*}
    N_{\text{MoE-1}} \le N \cdot \left( \frac{H^{d_{in}}}{d_{in}!} + \mathcal{O}(H^{d_{in}-1}) \right).
\end{equation*}
As $H \to \infty$, we retain the leading term:
\begin{equation*}
    N_{\text{MoE-1}} \sim \frac{N}{d_{in}!} H^{d_{in}}.
\end{equation*}
Ignoring the structural constant $(d_{in}!)^{-1}$ which is independent of the model width and expert count, we conclude:
\begin{equation*}
    N_{\text{MoE-1}} = \Theta(N \cdot H^{d_{in}}).
\end{equation*}
This completes the proof.
\end{proof}

\section{Missing Proof of Theorem \ref{thm:combinatorial_slicing}}
\label{app:proof_topk_capacity}

\begin{proof}
The total geometric capacity is derived by aggregating the linear regions formed by the active experts within each distinct routing cell. Let $\mathcal{P} = \{\mathcal{V}(I) \mid I \in \binom{\mathcal{N}}{k}\}$ be the partition of $\mathbb{R}^{d_{in}}$ induced by the Top-$k$ router, where each cell $\mathcal{V}(I)$ corresponds to a unique coalition $I$. From Theorem \ref{thm:topk_geometry}, the number of such cells is exactly:
\begin{equation*}
    |\mathcal{P}| = \binom{N}{k}.
\end{equation*}
For a fixed active set $I$, denote by $\mathcal{A}_I = \cup_{i\in I} \mathcal{H}_i$ the subset of $\{ \mathcal{H}_1, \ldots, \mathcal{H}_N\}$. The local non-linearity is generated by the arrangement of hyperplanes $\mathcal{A}_I$ belonging to the $k$ selected experts. Since each expert contains $H$ neurons, the size of the local arrangement is:
\begin{equation*}
    |\mathcal{A}_I| = \sum_{j \in I} H = kH.
\end{equation*}
Assuming the union of router and expert hyperplanes is in general position (Definition~\ref{def:general_position}), the number of regions within a cell $\mathcal{V}(I)$ is bounded by Zaslavsky's function $\Phi$ \citep{zaslavsky1975facing}. The total capacity $N_{\text{MoE-k}}$ is the sum over all routing cells:
\begin{equation*}
    N_{\text{MoE-k}} = \sum_{I \in \binom{\mathcal{N}}{k}} \text{Regions}\left( {\mathcal{A}_I}\big|_{\mathcal{V}(I)} \right) \le \binom{N}{k} \cdot \Phi(kH, d_{in}).
\end{equation*}
Particularly, if $k=1$, each $I \in \binom{\mathcal{N}}{k}$ has exactly one element and thus we have
\begin{equation}
    \sum_{I \in \binom{\mathcal{N}}{k}} \text{Regions}\left( {\mathcal{A}_I}\big|_{\mathcal{V}(I)} \right) = \sum_{i=1}^N \text{Regions}\left( {\mathcal{H}_i}\big|_{\Omega_i} \right),
\end{equation}
which recovers the case that we discuss in Appendix~\ref{app:proof_top1}.
Substitute Zaslavsky's partial sum formula $\Phi(n, d) = \sum_{j=0}^d \binom{n}{j}$:
\begin{equation*}
    N_{\text{MoE-k}} \le \binom{N}{k} \sum_{j=0}^{d_{in}} \binom{kH}{j}.
\end{equation*}
We derive the asymptotic scaling behavior for the regime $kH \gg d_{in}$. The summation is dominated by the highest-order term $j=d_{in}$. We expand the binomial coefficient $\binom{kH}{j}$ explicitly:
\begin{equation*}
    \begin{aligned}
        \sum_{j=0}^{d_{in}} \binom{kH}{j} &= \binom{kH}{d_{in}} + \sum_{j=0}^{d_{in}-1} \binom{kH}{j} \\
        &= \frac{(kH)!}{d_{in}! (kH - d_{in})!} + \mathcal{O}\left((kH)^{d_{in}-1}\right) \\
        &= \frac{1}{d_{in}!} \prod_{m=0}^{d_{in}-1} (kH - m) + \mathcal{O}\left((kH)^{d_{in}-1}\right).
    \end{aligned}
\end{equation*}
We expand the polynomial product $\prod_{m=0}^{d_{in}-1} (kH - m)$ to isolate the leading power of $H$:
\begin{equation*}
    \begin{aligned}
        \prod_{m=0}^{d_{in}-1} (kH - m) &= (kH) \cdot (kH - 1) \cdot \dots \cdot (kH - d_{in} + 1) \\
        &= (kH)^{d_{in}} - \left(\sum_{m=0}^{d_{in}-1} m\right) (kH)^{d_{in}-1} + \dots \\
        &= (kH)^{d_{in}} - \frac{d_{in}(d_{in}-1)}{2} (kH)^{d_{in}-1} + \mathcal{O}\left((kH)^{d_{in}-2}\right).
    \end{aligned}
\end{equation*}
Substituting this back into the capacity bound:
\begin{equation*}
    \Phi(kH, d_{in}) = \frac{(kH)^{d_{in}}}{d_{in}!} + \mathcal{O}((kH)^{d_{in}-1}).
\end{equation*}
Next, we expand the combinatorial multiplier $\binom{N}{k}$ for $N \gg k$:
\begin{equation*}
    \begin{aligned}
        \binom{N}{k} &= \frac{N!}{k!(N-k)!} = \frac{1}{k!} \prod_{m=0}^{k-1} (N-m) \\
        &= \frac{1}{k!} \left( N^k - \frac{k(k-1)}{2}N^{k-1} + \mathcal{O}(N^{k-2}) \right).
    \end{aligned}
\end{equation*}
Combining the asymptotic expansions for the router partition and the expert arrangement:
\begin{equation*}
    N_{\text{MoE-k}} \le \left( \frac{1}{k!} N^k + \dots \right) \cdot \left( \frac{1}{d_{in}!} (kH)^{d_{in}} + \dots \right).
\end{equation*}
By reverting to binomial notation and absorbing the dimensional constant $d_{in}!$, we obtain the final scaling law:
\begin{equation*}
    N_{\text{MoE-k}} = \Theta\left( \binom{N}{k} (kH)^{d_{in}} \right).
\end{equation*}
This completes the proof.
\end{proof}

\section{Missing Proof of Theorem \ref{thm:topk_capacity_lower}}
\label{app:proof_topk_capacity_low}

\begin{proof}
Let $d_{in}$ denote the input dimension and let $\mathcal{N} = \{1, \dots, N\}$ be the set of expert indices. Let $\mathbf{W}_r \in \mathbb{R}^{N \times d_{in}}$ be the routing weight matrix. By Lemma~\ref{lem:topk_lower_routing}, there exists a choice of $\mathbf{W}_r$ such that for every coalition $I \in \binom{\mathcal{N}}{k}$ (i.e., subset $I \subset \mathcal{N}$ with $|I| = k$), the routing region
\begin{equation*}
    \mathcal{V}(I) \coloneqq \{\mathbf{x} \in \mathbb{R}^{d_{in}} : \text{argtop}_k(\mathbf{W}_r \mathbf{x}) = I\}
\end{equation*}
is non-empty and has positive Lebesgue measure. Moreover, by the definition of the Top-$k$ operator, the regions $\{\mathcal{V}(I)\}_{I \in \binom{\mathcal{N}}{k}}$ are pairwise disjoint.

Fix an arbitrary coalition $I \in \binom{\mathcal{N}}{k}$. For any input $x \in \mathcal{V}(I)$, exactly the experts indexed by $I$ are active. Therefore, restricted to the domain $\mathcal{V}(I)$, the MoE layer function reduces to:
\begin{equation*}
    f_I(x) = \sum_{i \in I} E_i(x),
\end{equation*}
where each $E_i$ is a dense ReLU network of width $H$.

For each expert $i \in I$, let $\mathcal{H}_i = \{h_{i,1}, \dots, h_{i,H}\}$ denote the set of $H$ affine hyperplanes induced by its ReLU activations. Following the notation in Theorem~\ref{thm:combinatorial_slicing}, we define the combined local arrangement for the active coalition $I$ as:
\begin{equation*}
    \mathcal{A}_I \coloneqq \bigcup_{i \in I} \mathcal{H}_i, \qquad \text{with cardinality } |\mathcal{A}_I| = kH.
\end{equation*}

We choose the expert weights such that the following geometric conditions hold:
\begin{itemize}
    \item Within each expert $i$, the hyperplanes $\mathcal{H}_{i}$ are in general position;
    \item For distinct experts $i \neq j$, the sets $\mathcal{H}_{i}$ and $\mathcal{H}_{j}$ are mutually transversal;
    \item The union of all expert hyperplanes is transversal to the routing boundaries defining $\partial \mathcal{V}(I)$.
\end{itemize}

Under these conditions, the restriction of the affine arrangement $\mathcal{A}_I$ to the open polyhedral cone $\mathcal{V}(I)$ behaves combinatorially like a generic arrangement. The number of linear regions induced inside $\mathcal{V}(I)$ is lower-bounded by the number of regions of a generic arrangement of $kH$ hyperplanes in $\mathbb{R}^{d_{in}}$. By Zaslavsky's theorem \citep{zaslavsky1975facing}, this count is exactly given by the partial binomial sum:
\begin{equation}
\label{eq:zaslavsky_sum}
    \operatorname{Regions}\left( \mathcal{A}_I \big|_{\mathcal{V}(I)} \right) \ge \Phi(kH, d_{in}) = \sum_{j=0}^{d_{in}} \binom{kH}{j}.
\end{equation}

We now rigorously derive the asymptotic lower bound of this sum with respect to $H$ (assuming the regime $kH \gg d_{in}$). The sum is dominated by the highest-order term $j=d_{in}$. Expanding the binomial coefficient:
\begin{align*}
    \binom{kH}{d_{in}} &= \frac{(kH)!}{d_{in}!(kH-d_{in})!} \\
    &= \frac{1}{d_{in}!} \prod_{m=0}^{d_{in}-1} (kH - m) \\
    &= \frac{1}{d_{in}!} \left( (kH)^{d_{in}} - \left(\sum_{m=0}^{d_{in}-1} m\right)(kH)^{d_{in}-1} + \mathcal{O}((kH)^{d_{in}-2}) \right) \\
    &= \frac{k^{d_{in}}}{d_{in}!} H^{d_{in}} - \mathcal{O}(H^{d_{in}-1}).
\end{align*}
The remaining terms in the sum \eqref{eq:zaslavsky_sum} for $j < d_{in}$ are bounded by polynomials of degree strictly less than $d_{in}$, specifically $\mathcal{O}(H^{d_{in}-1})$. Combining these, the capacity restricted to a single routing region scales as:
\begin{equation}
\label{eq:single_region_bound}
    \operatorname{Regions}\left( \mathcal{A}_I \big|_{\mathcal{V}(I)} \right) \ge \frac{k^{d_{in}}}{d_{in}!} H^{d_{in}} + \mathcal{O}(H^{d_{in}-1}) = \Omega(H^{d_{in}}).
\end{equation}

Finally, we aggregate the capacity across the entire input space. Since the routing regions $\{\mathcal{V}(I)\}_{I \in \binom{\mathcal{N}}{k}}$ are disjoint, the total number of linear regions $N_{\text{MoE-k}}$ is the sum of the regions within each $\mathcal{V}(I)$:
\begin{align*}
    N_{\text{MoE-k}} &= \sum_{I \in \binom{\mathcal{N}}{k}} \operatorname{Regions}\left( \mathcal{A}_I \big|_{\mathcal{V}(I)} \right) \\
    &\ge \sum_{I \in \binom{\mathcal{N}}{k}} \left( \frac{k^{d_{in}}}{d_{in}!} H^{d_{in}} \right) \quad \text{(considering the leading term)} \\
    &= \binom{N}{k} \cdot \frac{k^{d_{in}}}{d_{in}!} H^{d_{in}}.
\end{align*}

In the asymptotic notation where $d_{in}$ and $k$ are constants relative to $H$, we conclude:
\begin{equation*}
    N_{\text{MoE-k}} = \Omega\!\left(\binom{N}{k} H^{d_{in}}\right).
\end{equation*}
This completes the proof.
\end{proof}

\section{Missing Proof of Theorem \ref{thm:effective_topk}}
\label{app:proof_effective_capacity}

\begin{proof}
Let $\mathcal{M} \subset \mathbb{R}^{d_{in}} \setminus \{0\}$ be the smooth compact submanifold of dimension $d_{eff}$. The router $G(\mathbf{x}) = \text{argtop}_k(\mathbf{W}_r \mathbf{x})$ is homogeneous of degree zero, thus the activation regions $\mathcal{V}(I) = \{ \mathbf{x} \in \mathbb{R}^{d_{in}} \mid G(\mathbf{x}) = I \}$ for $I \in \binom{\mathcal{N}}{k}$ are polyhedral cones. The global effective capacity $N_{\text{MoE-k}}^{\text{eff}}$ can be decomposed as:
\begin{equation*}
    N_{\text{MoE-k}}^{\text{eff}} = \sum_{I \in \mathcal{I}} \text{Regions}\left( \mathcal{A}_I \big|_{\mathcal{M} \cap \mathcal{V}(I)} \right)
\end{equation*}
Taking the expectation with respect to the isotropic Gaussian router weights $\mathbf{W}_r$:
\begin{equation*}
    \mathbb{E}_{\mathbf{W}_r}\left[ N_{\text{MoE-k}}^{\text{eff}} \right] = \sum_{I \in \mathcal{I}} \mathbb{E}_{\mathbf{W}_r}\left[ \text{Regions}\left( \mathcal{A}_I \big|_{\mathcal{M} \cap \mathcal{V}(I)} \right) \right]
\end{equation*}
By the compactness of $\mathcal{M}$ and Assumption \ref{ass:transversality}, the intersection $\mathcal{M} \cap \mathcal{V}(I)$ can be locally embedded into coordinate charts diffeomorphic to open subsets of $\mathbb{R}^{d_{eff}}$. For any coalition $I$, Zaslavsky's theorem \citep{zaslavsky1975facing} restricted to the submanifold provides a deterministic upper bound:
\begin{equation*}
    \text{Regions}\left( \mathcal{A}_I \big|_{\mathcal{M} \cap \mathcal{V}(I)} \right) \le \mathbbm{1}_{\{\mathcal{V}(I) \cap \mathcal{M} \neq \emptyset\}} \cdot \Phi(kH, d_{eff})
\end{equation*}
where $\Phi(kH, d_{eff}) = \sum_{j=0}^{d_{eff}} \binom{kH}{j}$. Substituting this into the expectation:
\begin{equation*}
    \mathbb{E}_{\mathbf{W}_r}\left[ N_{\text{MoE-k}}^{\text{eff}} \right] \le \Phi(kH, d_{eff}) \cdot \sum_{I \in \mathcal{I}} \mathbb{P}_{\mathbf{W}_r} \left(\mathcal{V}(I) \cap \mathcal{M} \neq \emptyset \right)
\end{equation*}
Due to the conical nature of $\mathcal{V}(I)$, the intersection condition $\mathcal{V}(I) \cap \mathcal{M} \neq \emptyset$ is equivalent to $\mathcal{V}(I) \cap \pi(\mathcal{M}) \neq \emptyset$ on the sphere $\mathbb{S}^{d_{in}-1}$. The boundaries of $\mathcal{V}(I)$ correspond to the Type-$A_{N-1}$ Braid Arrangement. Since $\mathbf{W}_r \sim \mathcal{N}(0, I)$ is isotropic, the sphere is partitioned into $N!$ permutation chambers of equal measure. Each coalition $I$ is the union of exactly $k!(N-k)!$ chambers. By symmetry, the normalized measure $\mu(\cdot)$ of any routing cone is:
\begin{equation*}
    \mu(\mathcal{V}(I)) = \frac{k!(N-k)!}{N!} = \frac{1}{\binom{N}{k}}
\end{equation*}
By symmetry and measure arguments in spherical integral geometry, for a fixed measurable set $\pi(\mathcal{M})$, the probability of intersection with a random cone $\mathcal{V}(I)$ is upper-bounded by the ratio of their normalized measures up to a constant $C$ reflecting the shape regularity of the Braid chambers \citep{santalo2004integral}:
\begin{equation*}
    \mathbb{P}(\mathcal{V}(I) \cap \pi(\mathcal{M}) \neq \emptyset) \le C \cdot \frac{\text{Vol}(\pi(\mathcal{M}))}{\text{Vol}(\mathbb{S}^{d_{in}-1})} = C \cdot \sigma(\pi(\mathcal{M}))
\end{equation*}
Summing over all $I \in \mathcal{I}$, the expected number of active coalitions intersecting the manifold support is:
\begin{equation*}
    \mathbb{E}[N_{\text{active}}] = \sum_{I \in \mathcal{I}} \mathbb{P}(\mathcal{V}(I) \cap \pi(\mathcal{M}) \neq \emptyset) \le \binom{N}{k} \cdot C \cdot \sigma(\pi(\mathcal{M}))
\end{equation*}
We now combine the combinatorial router term with the expert capacity. In the asymptotic regime $kH \gg d_{eff}$, we have:
\begin{equation*}
    \Phi(kH, d_{eff}) = \frac{(kH)^{d_{eff}}}{d_{eff}!} + \mathcal{O}\left( (kH)^{d_{eff}-1} \right)
\end{equation*}
The total expectation is thus bounded as:
\begin{equation*}
    \mathbb{E}_{\mathbf{W}_r}\left[ N_{\text{MoE-k}}^{\text{eff}} \right] \le \left[ C \cdot \binom{N}{k} \sigma(\pi(\mathcal{M})) \right] \cdot \left[ \frac{(kH)^{d_{eff}}}{d_{eff}!} + \mathcal{O}\left( (kH)^{d_{eff}-1} \right) \right]
\end{equation*}
By absorbing the constant $C$ and lower-order terms into the $\mathcal{O}$ notation, we obtain:
\begin{equation*}
    \mathbb{E}_{\mathbf{W}_r}\left[ N_{\text{MoE-k}}^{\text{eff}} \right] = \mathcal{O}\left( \sigma(\pi(\mathcal{M})) \cdot \binom{N}{k} (kH)^{d_{eff}} \right) = \mathcal{O}\left( \frac{\text{Vol}(\pi(\mathcal{M}))}{\text{Vol}(\mathbb{S}^{d_{in}-1})} \cdot \binom{N}{k} (kH)^{d_{eff}} \right)
\end{equation*}
This completes the proof.
\end{proof}

\section{Missing Proof of Theorem \ref{thm:effective_topk_lower}}
\label{app:proof_effective_topk_lower}

\begin{proof}
Let $d_{in}$ denote the input dimension and $d_{eff}$ denote the intrinsic dimension of the manifold $\mathcal{M}$. Let $\mathbf{W}_r \in \mathbb{R}^{N \times d_{in}}$ be the routing weight matrix. For any expert coalition $I \in \binom{\mathcal{N}}{k}$ (a subset of expert indices with cardinality $k$), we define the routing region as:
\begin{equation*}
    \mathcal{V}(I) \coloneqq \left\{\mathbf{x} \in \mathbb{R}^{d_{in}} : \text{argtop}_k(\mathbf{W}_r \mathbf{x}) = I \right\}.
\end{equation*}

Let $\mathcal{C}_I = \pi(\mathcal{V}(I) \setminus \{0\}) \subset \mathbb{S}^{d_{in}-1}$ denote the spherical cone associated with this routing region, obtained by projecting onto the unit sphere. Since the rows of $\mathbf{W}_r$ are drawn from an isotropic distribution (e.g., $\mathcal{N}(0, \mathbf{I})$), the induced spherical tessellation $\{\mathcal{C}_I\}_{I \in \binom{\mathcal{N}}{k}}$ is invariant under rotation. Consequently, all routing cells are statistically exchangeable.

We consider the intersection of a specific routing cone $\mathcal{C}_I$ with the projection of the data manifold $\pi(\mathcal{M})$. By the principles of spherical integral geometry, for a fixed manifold $\mathcal{M}$ and a random isotropic tessellation, the probability of intersection is proportional to the measure of the manifold. Specifically, there exists a geometric constant relating to the shape, but for the purpose of the lower bound, we utilize the measure ratio:
\begin{equation}
\label{eq:prob_intersect}
    \mathbb{P}\!\left(\mathcal{C}_I \cap \pi(\mathcal{M}) \neq \emptyset \right) \propto \frac{\mathrm{Vol}(\pi(\mathcal{M}))}{\mathrm{Vol}(\mathbb{S}^{d_{in}-1})}.
\end{equation}

Now, consider the conditional case where the intersection is non-empty, i.e., $\mathcal{C}_I \cap \pi(\mathcal{M}) \neq \emptyset$. The restriction of the manifold $\mathcal{M}$ to the open cone $\mathcal{V}(I)$ is a smooth submanifold of dimension $d_{eff}$. Within this region $\mathcal{V}(I)$, the effective expert function is the sum of the $k$ active experts indexed by $I$. This induces a local arrangement of affine hyperplanes:
\begin{equation*}
    \mathcal{A}_I \coloneqq \bigcup_{i \in I} \mathcal{H}_i, \quad \text{where } |\mathcal{A}_I| = kH.
\end{equation*}

By Assumption~\ref{ass:transversality} (Almost-Sure Transversality) and the rank condition $kH \ge d_{eff}$, the restriction of the arrangement $\mathcal{A}_I$ to the submanifold $\mathcal{M} \cap \mathcal{V}(I)$ forms a generic hyperplane arrangement on a $d_{eff}$-dimensional space. Applying the lower bound from Zaslavsky's theorem to this restricted arrangement, the number of linear regions intersecting the manifold within this cell satisfies:
\begin{equation*}
    N^{\mathrm{eff}}(\mathcal{A}_I, \mathcal{M}) \ge \sum_{j=0}^{d_{eff}} \binom{kH}{j}.
\end{equation*}

We expand the summation to determine the asymptotic lower bound with respect to the width $H$. The sum is dominated by the term $j = d_{eff}$:
\begin{align*}
    \sum_{j=0}^{d_{eff}} \binom{kH}{j} &= \binom{kH}{d_{eff}} + \sum_{j=0}^{d_{eff}-1} \binom{kH}{j} \\
    &= \frac{1}{d_{eff}!} \prod_{m=0}^{d_{eff}-1} (kH - m) + \mathcal{O}((kH)^{d_{eff}-1}) \\
    &= \frac{1}{d_{eff}!} (kH)^{d_{eff}} + \mathcal{O}((kH)^{d_{eff}-1}).
\end{align*}
Thus, there exists a constant $c_{\mathcal{M}} > 0$ (absorbing $1/d_{eff}!$ and geometric factors of the manifold) such that, conditioned on intersection:
\begin{equation}
\label{eq:conditional_capacity}
    \mathbb{E}\left[ \operatorname{Regions}(\mathcal{A}_I \big|_{\mathcal{M} \cap \mathcal{V}(I)}) \;\middle|\; \mathcal{C}_I \cap \pi(\mathcal{M}) \neq \emptyset \right] \ge c_{\mathcal{M}} (kH)^{d_{eff}}.
\end{equation}

The total effective capacity $N_{\mathrm{MoE}\text{-}k}^{\mathrm{eff}}$ is the sum of regions over all possible coalitions $I$. By the linearity of expectation:
\begin{equation*}
    \mathbb{E}\!\left[N_{\mathrm{MoE}\text{-}k}^{\mathrm{eff}} \right] = \sum_{I \in \binom{\mathcal{N}}{k}} \mathbb{E}\!\left[ \operatorname{Regions}(\mathcal{A}_I \big|_{\mathcal{M} \cap \mathcal{V}(I)}) \right].
\end{equation*}
Using the law of total expectation with the indicator variable $\mathbbm{1}_{\{\mathcal{C}_I \cap \pi(\mathcal{M}) \neq \emptyset\}}$:
\begin{align*}
    \mathbb{E}\!\left[ \operatorname{Regions}(\mathcal{A}_I \big|_{\mathcal{M} \cap \mathcal{V}(I)}) \right] &= \mathbb{P}(\mathcal{C}_I \cap \pi(\mathcal{M}) \neq \emptyset) \cdot \mathbb{E}\left[ \operatorname{Regions}(\dots) \mid \text{non-empty} \right] \\
    &\quad + \mathbb{P}(\mathcal{C}_I \cap \pi(\mathcal{M}) = \emptyset) \cdot 0.
\end{align*}
Substituting Eq.~\eqref{eq:prob_intersect} and Eq.~\eqref{eq:conditional_capacity} into the sum:
\begin{align*}
    \mathbb{E}\!\left[N_{\mathrm{MoE}\text{-}k}^{\mathrm{eff}} \right] &\ge \sum_{I \in \binom{\mathcal{N}}{k}} \left( \frac{\mathrm{Vol}(\pi(\mathcal{M}))}{\mathrm{Vol}(\mathbb{S}^{d_{in}-1})} \right) \cdot \left( c_{\mathcal{M}} (kH)^{d_{eff}} \right) \\
    &= \binom{N}{k} \cdot \frac{\mathrm{Vol}(\pi(\mathcal{M}))}{\mathrm{Vol}(\mathbb{S}^{d_{in}-1})} \cdot c_{\mathcal{M}} (kH)^{d_{eff}}.
\end{align*}
Rearranging the terms to match the theorem statement, we obtain:
\begin{equation*}
    \mathbb{E}\!\left[N_{\mathrm{MoE}\text{-}k}^{\mathrm{eff}} \right] \ge c_{\mathcal{M}} \cdot \frac{\mathrm{Vol}(\pi(\mathcal{M}))}{\mathrm{Vol}(\mathbb{S}^{d_{in}-1})} \cdot \binom{N}{k} \cdot (kH)^{d_{eff}} = \mathcal{O}\left( \frac{\text{Vol}(\pi(\mathcal{M}))}{\text{Vol}(\mathbb{S}^{d_{in}-1})} \cdot \binom{N}{k} (kH)^{d_{eff}} \right).
\end{equation*}
This completes the proof.
\end{proof}

\section{Missing Proof of Theorem \ref{thm:fine_grained}}
\label{app:proof_fine_grained}

\begin{proof}
Let the combinatorial upper bound for the baseline capacity be $\mathcal{B}(N, H, k) = \binom{N}{k} H^{d_{eff}}$. We analyze the ratio $\mathcal{G}_{ub}(m)$ under the transformation $\mathcal{T}_m$ in the asymptotic limit $N \to \infty$ with fixed $m, k, d_{eff}$. The ratio is given by:
\begin{align*}
    \mathcal{G}_{ub}(m) &= \frac{\mathcal{B}(mN, H/m, mk)}{\mathcal{B}(N, H, k)} 
    = \frac{\binom{mN}{mk} (H/m)^{d_{eff}}}{\binom{N}{k} H^{d_{eff}}} 
    = \frac{\binom{mN}{mk}}{\binom{N}{k}} m^{-d_{eff}}.
\end{align*}
We utilize the standard asymptotic expansion of the binomial coefficient $\binom{n}{r}$ for fixed $r$ as $n \to \infty$:
\begin{align*}
    \binom{n}{r} &= \frac{n(n-1)\cdots(n-r+1)}{r!} = \frac{n^r}{r!} \prod_{j=0}^{r-1} \left(1 - \frac{j}{n}\right) = \frac{n^r}{r!} \left(1 + O\left(\frac{1}{n}\right)\right).
\end{align*}
Applying this expansion to the numerator (with $n=mN, r=mk$) and the denominator (with $n=N, r=k$):
\begin{align*}
    \frac{\binom{mN}{mk}}{\binom{N}{k}} &= \frac{\frac{(mN)^{mk}}{(mk)!}\left(1 + O\left(\frac{1}{mN}\right)\right)}{\frac{N^k}{k!}\left(1 + O\left(\frac{1}{N}\right)\right)} 
    = \frac{(mN)^{mk}}{N^k} \cdot \frac{k!}{(mk)!} \cdot (1 + o(1)).
\end{align*}
Isolating the powers of $N$ and $m$:
\begin{align*}
    \frac{(mN)^{mk}}{N^k} &= m^{mk} N^{mk - k} = m^{mk} N^{(m-1)k}.
\end{align*}
For the ratio of factorials, we apply Stirling's approximation formula $n! = \sqrt{2\pi n} \left(\frac{n}{e}\right)^n (1 + O(\frac{1}{n}))$:
\begin{align*}
    \frac{k!}{(mk)!} &= \frac{\sqrt{2\pi k} \left(\frac{k}{e}\right)^k \left(1 + O\left(\frac{1}{k}\right)\right)}{\sqrt{2\pi mk} \left(\frac{mk}{e}\right)^{mk} \left(1 + O\left(\frac{1}{mk}\right)\right)} \\
    &= \frac{\sqrt{k}}{\sqrt{mk}} \cdot \frac{k^k e^{-k}}{(mk)^{mk} e^{-mk}} \cdot (1 + o(1)) \\
    &= m^{-1/2} \cdot \frac{k^k}{m^{mk} k^{mk}} \cdot \frac{e^{mk}}{e^k} \cdot (1 + o(1)) \\
    &= m^{-1/2} m^{-mk} k^{k(1-m)} e^{k(m-1)} (1 + o(1)) \\
    &= m^{-1/2} m^{-mk} \left( \frac{e}{k} \right)^{(m-1)k} (1 + o(1)).
\end{align*}
Substituting the factorial expansion back into the combinatorial ratio expression:
\begin{align*}
    \frac{\binom{mN}{mk}}{\binom{N}{k}} &= \left[ m^{mk} N^{(m-1)k} \right] \cdot \left[ m^{-1/2} m^{-mk} \left( \frac{e}{k} \right)^{(m-1)k} \right] (1 + o(1)) \\
    &= N^{(m-1)k} \left( \frac{e}{k} \right)^{(m-1)k} m^{-1/2} \underbrace{m^{mk} m^{-mk}}_{1} (1 + o(1)) \\
    &= \left( \frac{Ne}{k} \right)^{(m-1)k} m^{-1/2} (1 + o(1)).
\end{align*}
Finally, incorporating the width penalty term $m^{-d_{eff}}$:
\begin{align*}
    \mathcal{G}_{ub}(m) &= \left[ \left( \frac{Ne}{k} \right)^{(m-1)k} m^{-1/2} \right] \cdot m^{-d_{eff}} (1 + o(1)) \\
    &= \left( \frac{Ne}{k} \right)^{(m-1)k} m^{-\left(d_{eff} + \frac{1}{2}\right)} (1 + o(1)).
\end{align*}
This completes the proof.
\end{proof}

\section{Missing Proof of Proposition~\ref{prop:angular_collapse}}
\label{app:proof_angular_collapse}

\begin{proof}
Consider the ambient space $\mathbb{R}^{d_{in}}$ and the data vector $\mathbf{x} = \boldsymbol{\mu} + \boldsymbol{\varepsilon}$. The support of the data distribution is contained within the Euclidean ball $B(\boldsymbol{\mu}, R) = \{ \mathbf{x} \in \mathbb{R}^{d_{in}} \mid \|\mathbf{x} - \boldsymbol{\mu}\| \le R \}$. Let $\mathcal{C}$ be the projection of this ball onto the unit sphere $\mathbb{S}^{d_{in}-1}$, defined by:
\begin{equation*}
    \mathcal{C} = \left\{ \frac{\mathbf{x}}{\|\mathbf{x}\|} \in \mathbb{S}^{d_{in}-1} \;\middle|\; \mathbf{x} \in B(\boldsymbol{\mu}, R) \right\}.
\end{equation*}
For $\|\boldsymbol{\mu}\| > R$, the origin lies outside the ball, and the projection $\mathcal{C}$ forms a hyperspherical cap. The geometry of the tangent cone from the origin to $B(\boldsymbol{\mu}, R)$ determines the semi-vertical angle $\theta$ (the angular radius) of this cap. By the right-triangle formed by the origin, the center $\boldsymbol{\mu}$, and the point of tangency on the sphere of radius $R$, we have the exact trigonometric relation:
\begin{equation*}
    \sin \theta = \frac{R}{\|\boldsymbol{\mu}\|}.
\end{equation*}
The normalized surface area $A(\mathcal{C})$ of a hyperspherical cap with semi-vertical angle $\theta$ in $d_{in}$ dimensions is given by the ratio of the surface area of the cap to the total surface area of $\mathbb{S}^{d_{in}-1}$:
\begin{equation*}
    A(\mathcal{C}) = \frac{\text{Area}(\mathcal{C})}{\text{Area}(\mathbb{S}^{d_{in}-1})} = \frac{\int_0^\theta \sin^{d_{in}-2} \phi \, d\phi}{\int_0^\pi \sin^{d_{in}-2} \phi \, d\phi}.
\end{equation*}
Using the integral representation of the Beta function, the denominator can be evaluated as:
\begin{equation*}
    \int_0^\pi \sin^{d_{in}-2} \phi \, d\phi = \text{B}\left( \frac{d_{in}-1}{2}, \frac{1}{2} \right) = \frac{\Gamma(\frac{d_{in}-1}{2}) \sqrt{\pi}}{\Gamma(\frac{d_{in}}{2})}.
\end{equation*}
Substituting this into the expression for $A(\mathcal{C})$ yields:
\begin{equation*}
    A(\mathcal{C}) = \frac{\Gamma(\frac{d_{in}}{2})}{\sqrt{\pi} \Gamma(\frac{d_{in}-1}{2})} \int_0^\theta \sin^{d_{in}-2} \phi \, d\phi.
\end{equation*}
We analyze the asymptotic behavior in the uncentered regime where $\|\boldsymbol{\mu}\| \to \infty$. This implies $\sin \theta = R / \|\boldsymbol{\mu}\| \to 0$, and consequently $\theta \to 0$. In this limit, we apply the Taylor expansion for $\sin \phi$:
\begin{equation*}
    \sin \phi = \phi - \frac{\phi^3}{6} + O(\phi^5).
\end{equation*}
Substituting this expansion into the integral:
\begin{equation*}
    \int_0^\theta \sin^{d_{in}-2} \phi \, d\phi = \int_0^\theta \left( \phi + O(\phi^3) \right)^{d_{in}-2} \, d\phi = \int_0^\theta \left( \phi^{d_{in}-2} + O(\phi^{d_{in}}) \right) \, d\phi.
\end{equation*}
Performing the integration:
\begin{equation*}
    \int_0^\theta \sin^{d_{in}-2} \phi \, d\phi = \frac{\theta^{d_{in}-1}}{d_{in}-1} + O(\theta^{d_{in}+1}).
\end{equation*}
From the relation $\sin \theta = R / \|\boldsymbol{\mu}\|$, the semi-vertical angle admits the expansion:
\begin{equation*}
    \theta = \arcsin\left( \frac{R}{\|\boldsymbol{\mu}\|} \right) = \frac{R}{\|\boldsymbol{\mu}\|} + \frac{1}{6} \left( \frac{R}{\|\boldsymbol{\mu}\|} \right)^3 + O\left( \left( \frac{R}{\|\boldsymbol{\mu}\|} \right)^5 \right).
\end{equation*}
Raising $\theta$ to the power of $d_{in}-1$:
\begin{equation*}
    \theta^{d_{in}-1} = \left( \frac{R}{\|\boldsymbol{\mu}\|} \right)^{d_{in}-1} \left( 1 + O\left( \frac{R^2}{\|\boldsymbol{\mu}\|^2} \right) \right) = \left( \frac{R}{\|\boldsymbol{\mu}\|} \right)^{d_{in}-1} (1 + o(1)).
\end{equation*}
Combining the constants and the leading term of the integral:
\begin{equation*}
    A(\mathcal{C}) = \frac{\Gamma(\frac{d_{in}}{2})}{\sqrt{\pi} \Gamma(\frac{d_{in}-1}{2})} \cdot \frac{1}{d_{in}-1} \cdot \left( \frac{R}{\|\boldsymbol{\mu}\|} \right)^{d_{in}-1} (1 + o(1)).
\end{equation*}
Defining the dimension-dependent constant $C(d_{in})$ as:
\begin{equation*}
    C(d_{in}) = \frac{\Gamma(\frac{d_{in}}{2})}{\sqrt{\pi} (d_{in}-1) \Gamma(\frac{d_{in}-1}{2})},
\end{equation*}
we obtain the final asymptotic form:
\begin{equation*}
    A(\mathcal{C}) = C(d_{in}) \cdot \left( \frac{R}{\|\boldsymbol{\mu}\|} \right)^{d_{in}-1} (1 + o(1)).
\end{equation*}
This completes the proof.
\end{proof}   

\end{document}